\documentclass[twoside]{article}
\usepackage[a4paper]{geometry}
\usepackage[utf8]{inputenc}
\usepackage[T1]{fontenc} 
\usepackage{RR}
\usepackage{hyperref}

\usepackage{algorithm,algorithmic}
\usepackage{caption}
\usepackage{mathtools}
\usepackage{amsmath,amssymb}
\usepackage{amsthm}
\usepackage{graphicx}
\usepackage{breqn}
\usepackage{comment}
\usepackage{color}
\usepackage{cite}
\usepackage{breqn}
\newcommand{\mathbbop}[1]{\overline{\mathbb{#1}}}
\newtheorem{lemma}{Lemma}
\newtheorem{theorem}{Theorem}
\newtheorem{proposition}{Proposition}

\newcommand{\Gm}{\Gamma}
\newcommand{\nonum}{\nonumber}

\newcommand{\lal}{\lambda_{\alpha}}
\newcommand{\barsig}{\widehat{\sigma}}

\newcommand{\myeq}[1]{\stackrel{\mathclap{\scriptsize{\textnormal{#1}}}}{=}}

\usepackage{stackengine}

\newcommand{\textn}{\textnormal}
\usepackage{glossaries}
\setacronymstyle{long-short}

\newglossaryentry{modmat}%
{%
  name={$\boldsymbol{\mathcal{A}}$},
  text={\boldsymbol{\mathcal{A}}},
  description={Modularity matrix used for hypothesis testing},
}

\newglossaryentry{meanmodmat}%
{%
  name={$\overline{\boldsymbol{\mathcal{A}}}$},
  text={\overline{\boldsymbol{\mathcal{A}}}},
  description={Modularity matrix used for hypothesis testing},
}

\newglossaryentry{hyp0}%
{%
  name={$\mathcal{H}_0$},
  text={\mathcal{H}_0},
  description={Hypothesis 0 - The case of no embedded subgraph},
  sort={L}
}

\newglossaryentry{hyp1}%
{%
  name={$\mathcal{H}_1$},
  text={\mathcal{H}_1},
  description={Hypothesis 1 - The case where there is an embedded subgraph},
  sort={L}
}

\newglossaryentry{mean0}%
{%
  name={$\mu_{(0)}$},
  text={\mu_{(0)}},
  description={Mean of test statistic under Hypothesis 0},
  sort={L}
}

\newglossaryentry{mean1}%
{%
  name={$\mu_{(1)}$},
  text={\mu_{(1)}},
  description={Mean of test statistic under Hypothesis 1},
  sort={L}
}

\newglossaryentry{var0}%
{%
  name={$\sigma^2_{(0)}$},
  text={\sigma^2_{(0)}},
  description={Variance of test statistic under Hypothesis 0},
  sort={L}
}

\newglossaryentry{var1}%
{%
  name={$\mu_{(0)}$},
  text={\mu_{(0)}},
  description={Variance of test statistic under Hypothesis 1},
  sort={L}
}

\newglossaryentry{probbg}%
{%
  name={$p_b$},
  text={p_b},
  description={Edge probability of background graph},
  sort={L}
}

\newglossaryentry{probsg}%
{%
  name={$p_s$},
  text={p_s},
  description={Edge probability of the embedded subgraph},
  sort={L}
}

\newglossaryentry{domeigvecmodmat}%
{%
  name={$\mathbf{u}$},
  text={\mathbf{u}},
  description={Dominant eigenvector of modularity matrix $\gls{modmat}$},
  sort={L}
}

\newglossaryentry{domeigvecmeanmodmat}%
{%
  name={$\bar{\mathbf{u}}$},
  text={\bar{\mathbf{u}}},
  description={Dominant eigenvector of modularity matrix $\gls{meanmodmat}$},
  sort={L}
}

\newglossaryentry{spgaprat}%
{%
  name={$\Delta$},
  text={\Delta},
  description={The absolute ratio between the dominant eigenvalue and the edge of the spectrum},
}

\newacronym{clt}{CLT}{Central Limit Theorem}
\newacronym{wlog}{wlog}{without loss of generality}
\newacronym{whp}{whp}{with high probability}
\newacronym{iid}{i.i.d.}{independent and identically distributed}
\newacronym{as}{a.s.}{almost surely}
\newacronym{aas}{a.a.s}{asymptotically almost surely}
\newacronym{st}{s.t.}{such that}
\newacronym{rv}{rv}{random variable}
\newacronym{sbm}{SBM}{Stochastic Block Model}
\newacronym{rdpg}{RDPG}{Random Dot Product Graph}
\newacronym{er}{ER}{Erd\H{o}s-R\'enyi}
\newacronym{cdf}{CDF}{Cumulative Density Function}
\newacronym{tpm}{t.p.m}{Transition Probability Matrix}
\newacronym{ccdf}{CDF}{Complimentary Cumulative Density Function}
\newacronym{wrt}{w.r.t.}{with respect to}
\newglossaryentry{hasdistribution}%
{%
  name={$\sim$},
  text={\sim},
  description={Denotes has the distribution},
  sort={L}
}

\makeglossaries

\RRdate{Feb 2017}
\RRauthor{
Arun~Kadavankandy,
Konstantin~Avrachenkov,
Laura~Cottatellucci
and
~Rajesh~Sundaresan
}
\authorhead{K. Avrachenkov et al.}

\RRtitle{Le Pouvoir d'Information Supplementaire en Detection des Sousgraphes}
\RRetitle{The Power of Side-information in Subgraph Detection}
\titlehead{The Power of Side-information in Subgraph Detection}

\RRresume{Dans ce travail, nous nous attaquons au problème de la detection d'une communauté cachée dans un grand graphe. Nous considérons Belief Propagation (BP) pour cela, et le modele que nous utilisons est celui d'un petite graph Erd\H{o}s-R\'enyi qui est place dans un plus grande graphe, aussi Erd\H{o}s-R\'enyi. Nous assumons qu'il y a des informations supplémentaires de deux categories : la premiere o\`u certains sommets qui appartiennent \`a la communauté concernée sont connus, qui s’appellent des sommets indices (parfaites informations supplémentaires), et la deuxième o\`u certains sommets sont donnés, mais qu'une partie d'eux ne sont correctes (informations supplémentaires imparfaites). L'algorithme bas\'e sur Belief propagation que nous présentons peut fonctionner en presence de ces deux categories d'information supplémentaires, et en l'analysant nous prouvons que meme une très peu de leur quantité enlève un seuil de « detectabilit\'e »  qui existe dans Belief Propagation sans aucunes informations supplémentaires. Nous fournissons aussi des résultats d'experiences conduites sur des graphes synthétiques et aussi deux graphes reels. 
}
\RRabstract{
In this work, we tackle the problem of hidden community detection. We consider Belief Propagation (BP) applied to the problem of detecting a hidden Erd\H{o}s-R\'enyi (ER) graph embedded in a larger and sparser ER graph, in the presence of side-information. We derive two related algorithms based on BP to perform subgraph detection in the presence of two kinds of side-information. The first variant of side-information consists of a set of nodes, called cues, known to be from the subgraph. The second variant of side-information consists of a set of nodes that are cues with a given probability. It was shown in past works that BP without side-information fails to detect the subgraph correctly when an effective signal-to-noise ratio (SNR) parameter falls below a threshold. In contrast, in the presence of non-trivial side-information, we show that the BP algorithm achieves asymptotically zero error for any value of the SNR parameter. We validate our results through simulations on synthetic datasets as well as on a few real world networks.}
\RRmotcle{Belief Propagation, Graphes Al\'eatoires non-ori\'ent\'es, Detection de Commaunautes, Stochastic Block Model, Methode de Cavite, Seuil de Detectabilite}
\RRkeyword{Belief Propagation, Dense community detection, Cavity Method, Detectability, Stochastic Block Model}
\RRprojet{Neo}

\RCSophia 

\begin{document}
\RRNo{8974}

\makeRR   


\section{Introduction}
\subsection{Problem Motivation}
We consider the problem of hidden community detection in graphs in the presence of side-information. In various disciplines graphs have been used to model, in a parsimonious fashion, relationships between heterogenous data. The presence of a dense hidden community in such graphs is usually indicative of interesting phenomena in the associated real-world network.

An example application of dense subgraph detection in Signal Processing is the problem of Correlation Mining \cite{firouzi2013predictive}. Given a network of correlated signals, a graph is formed with nodes representing signals, and weighted links representing pairwise correlations. The problem of detecting a group of closely correlated signals is then a dense subgraph detection problem on the constructed graph \cite{firouzi2013predictive}. Dense subgraph detection also finds application in real-world computer and social networks; for e.g., in detecting fraudulent activity \cite{chau2006detecting,beutel2013copycatch,smith2014bayesian}. It can, in addition, be viewed as a signal recovery problem on graphs \cite{chen2015signal,wang2015local}.

A majority of subgraph detection algorithms try to find a subset of nodes that maximizes some objective such as the average link density within the subset\cite{lee2010survey}. A good way to benchmark the performance of various community detection algorithms is to validate them on generative graph models with inherent community structure. In this work, we model the hidden community as a small but well-connected Erd\H{o}s-R\'enyi graph embedded within a larger but sparser Erd\H{o}s-Renyi graph. This model was used in \cite{Mifflin2004} to capture terrorist transactions in a computer network.  It is a special case of the Stochastic Block Model (SBM), which has been widely used to assess the performance of different community detection algorithms \cite{rohe2011spectral}.

The study of subgraph detection on generative models is interesting in itself from an algorithmic perspective. Recent works on hidden community detection and related problems demonstrate the presence of sharp phase transitions in the range of parameter values between three regimes: easy (detection achievable with relatively small computational costs), hard (computationally taxing, but detectable), and impossible to detect\cite{hajek2015recovering,montanari2015finding,caltagirone:hal-01391609}. We provide more details on these phenomena while reviewing prior works in the next subsection. The novel aspect of this paper is a theoretical study of the impact of side-information on this computational barrier. The form of side-information we consider is the identity of special nodes called cues that are known to belong to the subgraph, either deterministically or with some level of certainty. One often has access to such prior knowledge in real-world applications\cite{SokolAGM12,zhou2004learning,zhu2003semi}.

By developing and analyzing the asymptotic performance of a local algorithm based on Belief Propagation (BP), we show that even a small amount of side-information can lead to the disappearance of the computational barrier. BP is an efficient way to perform approximate ML detection on certain types of graphs using distributed and local message passing \cite{mezard2009information}. It belongs to the class of guilt-by-association schemes \cite{koutra2011unifying} and has been successfully applied to many practical problems in graphs such as fraud detection \cite{chau2006detecting} and data mining\cite{kang2011mining}.

\subsection{Previous works}


 Consider a graph with $n$ nodes that contains a hidden community of size $K.$ The edge probability between any two nodes within the community is $p$ and it is $q$ otherwise, such that $p>q.$ The parameters $p, q$ and $K$ can in general be functions of $n.$ This model, denoted by $G(K,n,p,q),$ was already considered in \cite{Mifflin2004a,Miller2010,arun2016} in the context of anomaly detection.

A special case of the above model is the hidden clique model with $p = 1$ and $q = 1/2.$ The study of clique detection algorithms demonstrate the presence of phase transitions in the subgraph size $K$ between impossible, hard and easy regimes. If $K\le 2(1-\epsilon) \log_2(n),$ the clique is impossible to detect; however, an exhaustive search detects the clique nodes when $K \ge 2(1+\epsilon)\log_2(n).$ In contrast, the smallest clique size that can be detected in polynomial time is believed to be $c\sqrt{n}$\cite{alon1998finding} for some $c > 0,$ and the minimum clique-size that can be detected in nearly-linear time is believed to be $\sqrt{n/e}$\cite{deshpande2015finding}.

The computational barriers for subgraph detection in a sparse graph without cues were studied in \cite{montanari2015finding,hajek2015recovering,hajek2015information}.
In \cite{montanari2015finding} the author investigated the performance of Maximum Likelihood (ML) detection and BP, and analyzed the phase transition with respect to an effective signal-to-noise ratio (SNR) parameter $\lambda$ defined as
\begin{equation}
  \label{eSNR}
  \lambda = \frac{K^2(p-q)^2}{(n-K)q}.
\end{equation}
The larger the $\lambda,$ the easier it is to detect the subgraph. Subgraph recovery was considered under a parameter setting where $K = \kappa n, p = a/n$ and $q = b/n,$ where $\kappa,a$ and $b$ are constants independent of $n.$ It was shown under this setting that, for any $\lambda > 0,$ an exhaustive search can detect the subgraph with success probability approaching one as $\kappa \to 0.$ However BP, which has quasi-linear time complexity, achieves non-trivial success probability only when $\lambda > 1/e$ in the same regime. Further, for $\lambda<1/e,$ the success probability of the algorithm is bounded away from one. This demonstrates the existence of a computational barrier for local algorithms.

In \cite{hajek2015information} the authors show that when $K = o(n),$ i.e., when $\kappa \to 0,$ and $p, q$ are such that $a = np = n^{o(1)}$ and $p/q = O(1),$ ML detection succeeds when $\lambda = \Omega(\frac{K}{n}\log(\frac{n}{K})),$ i.e., detection is possible even when the SNR parameter goes to zero so long as it does not go to zero too fast. Under the same parameter setting, it was shown that BP succeeds in detecting the subgraph with the fraction of misdetected nodes going to zero, only when $\lambda > 1/e$ \cite{hajek2015recovering}. Therefore, $\lambda = 1/e$ represents a computational barrier for BP in the subgraph detection problem without side-information.

In the present work, we examine the impact of side-information on the above computational barrier. To the best of our knowlege, ours is the first theoretical study of the performance of local algorithms for subgraph detection in the presence of side-information in $G(K,n,p,q).$ In \cite{miller2015residuals}, the authors compared, but only empirically, several guilt-by-association schemes for subgraph detection with cues.

There exist many works on the effect of side-information in the context of identifying multiple communities\cite{allahverdyan2010community,caltagirone:hal-01391609,cai2016inference,Mossel2016}. These works considered a different variant of the SBM where nodes are partitioned into two or more communities, with dense links inside communities and sparse links across communities. The authors of \cite{cai2016inference} and \cite{Mossel2016} consider a BP algorithm to detect two equal-sized communities. In \cite{Mossel2016}, the side-information is such that all nodes indicate their community information after passing it through a binary symmetric channel with error rate $\alpha.$ They show that when $\alpha<1/2,$ i.e., when there is non-trivial side-information, there is no computational barrier and BP works all the way down to the detectability threshold called the Kesten-Stigum threshold\cite{abbe2015detection}. In \cite{cai2016inference}, a vanishing fraction $n^{-o(1)}$ of nodes reveal their true communities. Again, there is no computational barrier and BP works all the way down to the detectability threshold. A fuller picture is available in \cite{caltagirone:hal-01391609}, which considers asymmetric communities and asymmetric connection probabilities within communities. In this setting, the authors of \cite{caltagirone:hal-01391609} demonstrate the presence of all three regimes (easy to detect, hard to detect but possible via exhaustive search, and impossible to detect) as a function of the size of the smallest community. In contrast, \cite{Mossel2016} and \cite{cai2016inference} consider equal-sized communities with the same edge probability within each community. In \cite{caltagirone:hal-01391609,cai2016inference,Mossel2016}, the parameters are chosen such that node degrees alone are not informative. Our work is different from the above settings, in that we deal with a single community, and the degrees can be informative in revealing node identities, i.e., the average degree of a node within the subgraph $Kp+(n-K)q$ is greater than $nq,$ the average degree of a node outside the subgraph. In this setting we show that the computational barrier disappears when side-information is available. We emphasize that our results cannot be obtained as a special case of the results in \cite{allahverdyan2010community,caltagirone:hal-01391609,cai2016inference,Mossel2016}.


\subsection{Summary of Results}
We consider subgraph detection in $G(K,n,p,q)$ with two types of side-information:
\begin{enumerate}
  \item A fraction $\alpha$ of subgraph nodes are revealed to the detector, which we call reliable cues. This represents the case of perfect side-information.
  \item A similar number of nodes are marked as cues, but they are unreliable, i.e., imperfect side-information.
\end{enumerate}
These two types of side-information are typical in semi-supervised clustering applications\cite{SokolAGM12,zhou2004learning,zhu2003semi}.

We use BP for subgraph detection to handle these two kinds of side-information. Our computations are local and distributed and require only neighbourhood information for each node in addition to the graph parameters $p,q$ and $K.$ 

We analyze the detection performance of our algorithm when $p = a/n, q = b/n$ with $a,b$ fixed and $K = \kappa n$ with $\kappa$ fixed, as in the regime of \cite{montanari2015finding}. Under this setting, we derive recursive equations for the distributions of BP messages in the limit as the graph size $n$ tends to infinity. These recursions allow for numerical computation of the error rates for finite values of $a,b$ and $\kappa$.

Based on these recursions, we obtain closed form expressions for the distributions when $a,b \to \infty.$ We then show that when there is non-trivial side-information, the expected fraction of misclassified nodes goes to zero as $\kappa \to 0,$ for any positive value of the respective SNR parameter $\lambda_{\alpha}$ or $\lambda,$ for perfect or imperfect side-information, made explicit later. Thus the computational barrier of $\lambda = 1/e$ for BP without side-information disappears when there is side-information.



We validate our theoretical findings by simulations. To demonstrate the practical usefulness of our algorithm we also apply it to subgraph detection on real-world datasets.

The algorithm for imperfect side-information with its numerical validation on synthetic datasets was submitted for review to ISIT 2017\cite{akadISIT}. The rest of the material, such as the algorithm for perfect side-information, all the proofs and numerical results on real-world datasets, is new in this journal version.

\subsection{Organization}
The rest of the paper is organized as follows. In Subsection \ref{sec:not} we delineate useful notation. In Section \ref{sec:model} we describe the model and define the problem in detail. In Section \ref{sec:bpalgo}, we present our algorithm with perfect cues and explain the steps in its derivation.  In Section \ref{sec:errornan} we derive the asymptotic distribution of BP messages. In particular, in section \ref{sec:detmethod}, we prove our main result on the asymptotic error rate of our algorithm. In Section \ref{impsibp} we present our algorithm with imperfect side-information and provide a result on its asymptotic error rate. In Section \ref{sec:sims} we present results on our experiments on the synthetic graph as well as a few real-world graphs. In Section \ref{sec:con}, we conclude with some suggestions for future work. Some proofs are relegated to supplementary material for lack of space.

\subsection{Notation and Nomenclature} \label{sec:not} A graph node is denoted by a lower case letter such as $i.$ The graph distance between two nodes $i$ and $j$ is the length of the shortest sequence of edges to go from $i$ to $j.$ The neighbourhood of a node $i,$ denoted by $\delta i$ is the set of one-hop neighbours of $i,$ i.e., nodes that are at a graph distance of one. Similarly, we also work with $t$-hop neighbours of $i,$ denoted as $G_i^t,$ the set of nodes within a distance of $t$ from $i.$ Note that $G_i^1 = \delta i.$  We use the following symbols to denote set operations: $C = A\backslash B$ is the set of elements that belong to $A$ and not $B$ and $\Delta$ denotes  the set difference, i.e., $ A\Delta B = (A\cup B) \backslash (A \cap B).$ Also $|C|$ denotes the cardinality of the set $C.$ The indicator function for an event $A$ is denoted by $\mathbf 1(A),$ i.e., $\mathbf 1(A) = 1$ if $A$ is true and 0 otherwise.  The symbol $\sim$ denotes the distribution of a \gls{rv}, for example $X \sim \text{Poi}(\gamma)$ means that $X$ is a Poisson distributed \gls{rv} with mean $\gamma.$ Also, $\mathcal N(\mu,\sigma^2)$ denotes the Gaussian distribution with mean $\mu$ and variance $\sigma^2.$ The symbol $\xrightarrow{D}$ denotes convergence in distribution.

\section{Model and Problem Definition}\label{sec:model}
Let $G(K,n,p,q)$ be a random undirected graph with $n$ nodes and a hidden community $S$ such that $|S| \quad= K.$ Let $\mathcal G = (V,E)$ be a realization of $G(K,n,p,q).$ An edge between two nodes appears independently of other edges such that $\mathbb P ((i,j) \in E |i,j \in S) = p$ and $\mathbb P ((i,j) \in E |i \in S, j \not\in S) = \mathbb P ((i,j) \in E |i,j \not\in S) = q.$ We assume that $S$ is chosen uniformly from $V$ among all sets of size $K.$ Additionally let $p = a/n$ and $q = b/n,$ where $a$ and $b$ are constants independent of $n.$ Such graphs, with average degree $O(1),$ are called diluted graphs. We use a function $\sigma : V \to \{0,1\}^n$ to denote community membership such that $\sigma_i = 1$ if $i \in S$ and $0$ otherwise.
Next we describe the model for selecting $C,$ the set of cues.
To indicate which nodes are cues, we introduce a function $c: V \to \{0,1\}^n$ \gls{st} $c_i = 1$ if $i$ is a cued vertex and $c_i = 0$ otherwise. The model for cues depends on the type of side-information: perfect or imperfect.

The side-information models are as follows:
\begin{enumerate}
  \item \textbf{Perfect side-information}: In this case the cues are reliable, i.e., they all belong to the subgraph. To construct $C$ we sample nodes as follows
  \[
  \mathbb P (c_i = 1|\sigma_i = x) =
  \begin{cases}
    \alpha & \text{ if } x = 1 \\
     0 & \text{ if } x = 0,
  \end{cases}
  \]
  for some $\alpha \in (0,1).$ Under this model we have
  \begin{align}
    n \mathbb P(c_i = 1) &= \sum_{i\in V}\mathbb P(c_i = 1|\sigma_i=1)\mathbb P(\sigma_i = 1)\nonum \\
    &= \alpha K.     \label{eq:alpha}
  \end{align}
  \item \textbf{Imperfect side-information}: Under imperfect side-information, the cues are unreliable.
We generate $C$ by sampling nodes from $V$ as follows using a fixed $\beta \in (0,1].$
For any $i \in V$:
  \begin{equation}\label{eq:samimp}
  \mathbb P\left(c_i = 1|\sigma_i = x\right) =
    \begin{cases}
    \alpha \beta & \text{ if } x = 1,\\
    \frac{\alpha K (1-\beta)}{(n-K)} & \text{ if } x = 0.
    \end{cases}
  \end{equation}
  Under this model we have for any $i \in V,$
  \begin{align*}
\mathbb P (c_i=1)
&= \mathbb P(\sigma_i=1)\mathbb P(c_i=1|\sigma_i = 1) \\
&\quad +  \mathbb P(\sigma_i=0)\mathbb P(c_i=1|\sigma_i = 0) \\
&=  \frac{K}{n} \alpha \beta +  \frac{(n-K)}{n} \frac{\alpha K (1-\beta)}{(n-K)}\\
&= \alpha K/n;
  \end{align*}
  hence it matches with (\ref{eq:alpha}) of the perfect side-information case.
It is easy to verify that under the above sampling
\begin{equation}\label{def:beta}
\mathbb{P}\left(\sigma_i = 1| c_i = 1\right) = \beta,
\end{equation}
 which provides us with the interpretation of $|\log(\beta/(1-\beta))|$ as a reliability parameter for cue information.

\end{enumerate}

Given $G,C$ our objective is to infer the labels $\{\sigma_i, i \in V\backslash C\}.$
The optimal detector that minimizes the expected number of misclassified nodes is the per-node MAP detector given as\cite{hajek2015information}:
\[
\hat{\sigma}_i = \mathbf{1}\left(R_i > \log\frac{\mathbb P (\sigma_i=0)}{\mathbb P (\sigma_i=1)}\right),
\]
where
\[
R_i = \log \left(\frac{\mathbb P (G,C|\sigma_i = 1)}{\mathbb P (G,C|\sigma_i = 0)}\right)
\]
is a log-likelihood ratio of the detection problem.
Observe that this detector requires the observation of the whole graph. Our objective then is to compute $R_i$ for each $i$ using a local Belief Propagation (BP) algorithm and identify some parameter ranges for which it is useful. Specifically, we want to show that a certain barrier that exists for BP when $\alpha = 0$ disappears when $\alpha \beta > 0.$




\section{Belief Propagation Algorithm for Detection with Perfect Side-information}\label{sec:bpalgo}
In this section we present the BP algorithm, Algorithm~\ref{alg:bp_alg}, which performs detection in the presence of perfect side-information.
We provide here a brief overview of the algorithm. At step $t$ of Algorithm \ref{alg:bp_alg}, each node $u \in V \backslash C$ updates its own log-likelihood ratio based on its $t$-hop neighbourhood:
\begin{equation}\label{eq:bpb}
R_u^t := \log \left(\frac{\mathbb P (G_u^t,C_u^t|\sigma_u = 1)}{\mathbb P (G_u^t,C_u^t|\sigma_u = 0)}\right),
\end{equation}
where $G_u^t$ is the set of $t$-hop neighbours of $u$ and $C_u^t$ is the set of cues in $G_u^t,$ i.e., $C_u^t = G_u^t \cap C.$ The beliefs are updated according to (\ref{eq:bpbelief}).
The messages transmitted to $u$  by the nodes $i \in \delta u,$ the immediate neighbourhood of $u,$ are given by
\begin{equation}\label{eq:bpm}
R_{i \to u}^{t} := \log \left(\frac{\mathbb P (G_i^t\backslash u,C_i^t\backslash u|\sigma_i = 1)}{\mathbb P (G_i^t\backslash u,C_i^t\backslash u|\sigma_i = 0)}\right),
\end{equation}
where $G_i^t\backslash u$ and $C_i^t\backslash u$ are defined as above, but excluding the contribution from node $u.$ Node $i$ updates $R_{i \to u}^{t}$ by acquiring messages from its neighbours, except $u,$ and aggregating them according to (\ref{eq:bpupdate}). If node $u$ is isolated, i.e., $\delta u = \emptyset,$ there are no updates for this node. It can be checked that the total computation time for $t_f$ steps of BP is $O(t_f|E|).$

\begin{algorithm}
  \caption{BP with perfect side-information}
  \label{alg:bp_alg}
  \begin{algorithmic}[1]
    \STATE Initialize: Set $R^0_{i \to j}$ to 0, for all $(i,j) \in E$ with $i,j \not\in C.$ Let $t_f < \frac{\log(n)}{\log(np)}+1.$ Set $t=0.$
    \STATE For all directed pairs $(i,u) \in E,$ such that $i,u \notin C$:
    \begin{dmath}
\label{eq:bpupdate}
    R_{i \to u}^{t+1} = -K(p-q) + \sum_{l \in C^1_i, l \neq u} \log\left(\frac{p}{q}\right) +  \sum_{l \in \delta i \backslash C_i^1,l \neq  u} \log \left(\frac{\exp(R^t_{l \to i}-\upsilon) (p/q) +1}{\exp(R_{l \to i}^t - \upsilon) + 1}\right),
\end{dmath}
where $\upsilon = \log(\frac{n-K}{K(1-\alpha)}).$
    \label{op1}
    \STATE Increment $t,$ if $t< t_f-1$ go back to \ref{op1}, else go to \ref{op2}
    \STATE Compute $R_u^{t_f}$ for every $u \in V\backslash C$ as follows: \label{op2}
      \begin{dmath}
\label{eq:bpbelief}
    R_u^{t+1} = -K(p-q) + \sum_{l \in C_u^1} \log\left(\frac{p}{q}\right) + \sum_{l \in \delta u \backslash C_u^1} \log \left(\frac{\exp(R^t_{l \to u}-\upsilon) (p/q) +1}{\exp(R_{l \to u}^t - \upsilon)+ 1}\right)
\end{dmath}
    \STATE The output set is the union of $C$ and the $K-|C|$ set of nodes in $V \backslash C$ with the largest values of $R_u^{t_f}.$

      \end{algorithmic}
\end{algorithm}
The detailed derivation of the algorithm can be found in Appendix \ref{ap:gwtree}. The derivation consists of two steps. First we establish a coupling between $G_u^t,$ the $t$-hop neighbourhood of a node $u$ of the graph and a specially constructed Galton-Watson (G-W) tree\footnote{Detailed in Appendix \ref{ap:gwtree}} $T_u^t$ of depth $t$ rooted on $u.$ This coupling ensures that for a carefully chosen $t = t_f$ the neighbourhood $G_u^{t_f}$ of the node is a tree with probability tending to one as $n \to \infty$ (i.e., with high probability (w.h.p)). The second step of the derivation involves deriving the recursions (\ref{eq:bpupdate}) and (\ref{eq:bpbelief}) to compute (\ref{eq:bpm}) and (\ref{eq:bpb}) respectively, using the tree coupling.

The output of the algorithm is $C$ along with the set of $K - |C|$ nodes with the largest value of log-likelihoods $R_i^{t_f}.$
In the following section we derive the asymptotic distributions of the BP messages as the graph size tends to infinity, so as to quantify the error performance of the algorithm.

\section{Asymptotic Error Analysis}\label{sec:errornan}
In this section we analyze the distributions of BP messages $R_{i\to u}^t$ given $\{\sigma_i = 1\}$ and given $\{\sigma_i = 0\}$ for $i \in V \backslash C.$ First, we derive a pair of recursive equations for the asymptotic distributions of the messages $R_{i\to u}^t$ given $\{\sigma_i = 0,c_i=0\}$ and given $\{\sigma_i = 1,c_i=0\}$ in the limit as $n \to \infty$ in Lemma \ref{l:recdistequation}. In Proposition \ref{prop:largedegreeasyms} we present the asymptotic distributions of the messages in the large degree regime where $a,b \to \infty.$ This result will enable us to derive the error rates for detecting the subgraph in the large degree regime (Theorem \ref{pr:weakrecovery}). Finally, we contrast this result with Proposition \ref{thm:monta} from \cite{montanari2015finding}, which details the limitation of local algorithms.


Instead of studying $R_{i\to u}^t$ directly, we look at the log-likelihood ratios of the posterior probabilities of $\sigma_i$ given as
\[
\widetilde{R}_i^t = \log\left(\frac{\mathbb{P}(\sigma_i =1|G_i^t,C_i^t,c_i=0)}{\mathbb{P}(\sigma_i =0|G_i^t,C_i^t,c_i=0)}\right)
\]
and the associated messages $\widetilde R_{i\to u}^t.$
By Bayes rule, $\widetilde{R}_{i\to u}^t = R_{i\to u}^t - \upsilon,$ where
\[
\upsilon = \log \left(\frac{\mathbb{P}(\sigma_i = 0|c_i = 0)}{\mathbb{P}(\sigma_i = 1|c_i = 0)}\right) =  \log\left(\frac{n-K}{K(1-\alpha)}\right).
\]

Let $\xi_0^{t},\xi_1^{t}$ be \gls{rv}s with the same distribution as the messages $\widetilde R_{i\to u}^{t}$ given $\{\sigma_i = 0, c_i = 0\}$ and given $\{\sigma_i = 1, c_i = 0\},$ respectively in the limit as $n \to \infty.$ Based on the tree coupling in Lemma \ref{l:coupling} of Appendix \ref{ap:gwtree}, it can be shown that these \gls{rv}s satisfy the recursive distributional evolutionary equations given in the following lemma.
\begin{lemma}\label{l:recdistequation}
The random variables $\xi_0^t$ and $\xi_1^t$ satisfy the following recursive distributional equations with initial conditions $\xi_0^0 = \xi_1^0 = \log\left({\kappa(1-\alpha)/(1-\kappa)}\right).$
\begin{align}
\xi_{0}^{(t+1)} &\myeq{D} h + \sum_{i=1}^{L_{0c}}\log(\rho) + \sum_{i=1}^{L_{00}}f(\xi_{0,i}^{(t)}) + \sum_{i=1}^{L_{01}}f(\xi_{1,i}^{(t)}) \label{eq:d1}\\
\xi_{1}^{(t+1)} &\myeq{D} h + \sum_{i=1}^{L_{1c}}\log(\rho)  + \sum_{i=1}^{L_{10}}f(\xi_{0,i}^{(t)})+ \sum_{i=1}^{L_{11}}f(\xi_{1,i}^{(t)}),\label{eq:d2}
\end{align}
where $\myeq{D}$ denotes equality in distribution, $h = -\kappa(a - b) - \upsilon,$ $ \rho := {p}/{q} = a/b, $ and the function $f$ is defined as
\begin{equation}\label{eq:fx}
f(x) := \log\left(\frac{\exp( x) \rho + 1}{\exp(x)+1} \right).
\end{equation}
The rvs $\xi_{0,i}^t, i = 1,2,\ldots$ are independent and identically distributed (iid) with the same distribution as $\xi_0^t.$ Similarly $\xi_{1,i}^t, i = 1,2, \ldots$ are iid with the same distribution as $\xi_1^t.$ Furthermore, $L_{00} \sim \text{Poi}((1-\kappa) b), L_{01} \sim \text{Poi}(\kappa b(1-\alpha)), L_{10} \sim \text{Poi}((1-\kappa)b), L_{11} \sim \text{Poi}(\kappa a(1-\alpha)), L_{0c} \sim \text{Poi}(\kappa b\alpha)$ and $L_{1c} \sim \text{Poi}(\kappa p\alpha).$
\end{lemma}
\begin{proof}This follows from (\ref{eq:bpupdate}) and the tree coupling in Lemma \ref{l:coupling} of Appendix \ref{ap:gwtree}.\end{proof}

We define the effective SNR for the detection problem in the presence of perfect side-information as:
\begin{equation}\label{eq:effsnr}
\lambda_{\alpha} = \frac{K^2(p-q)^2(1-\alpha)^2}{(n-K)q} = \frac{\kappa^2(a-b)^2(1-\alpha)^2}{(1-\kappa)b},
\end{equation}
where the factor $(1-\alpha)^2$ arises from the fact that we are now trying to detect a smaller subgraph of size $K(1-\alpha).$


We now present one of our main results, on the distribution of BP messages in the limit of large degrees as $a,b \to \infty$ such that $\lambda_{\alpha}$ is kept fixed.
\begin{proposition}
\label{prop:largedegreeasyms}
In the regime where  $\lambda_{\alpha}$ and $\kappa$ are held fixed and $a,b \to \infty,$ we have
\[
\xi_0^{t+1} \xrightarrow{D} \mathcal{N}\left(-\log\frac{1-\kappa}{\kappa(1-\alpha)} - \frac{1}{2}\mu^{(t+1)},\mu^{(t+1)}\right)
\]
\[
\xi_1^{t+1} \xrightarrow{D} \mathcal{N}\left(-\log\frac{1-\kappa}{\kappa(1-\alpha)} + \frac{1}{2}\mu^{(t+1)},\mu^{(t+1)}\right).
\]
The variance $\mu^{(t)}$ satisfies the following recursion with initial condition $\mu^{(0)} = 0:$
\begin{dmath}\label{eqrecu}
  \mu^{(t+1)} = \lambda_{\alpha}\alpha\frac{1-\kappa}{(1-\alpha)^2\kappa} + \lambda_{\alpha} \mathbb{E} \left ( \frac{(1-\kappa)}{\kappa (1-\alpha) + (1-\kappa) \exp(-\mu^{(t)}/2  - \sqrt{\mu^{(t)}} Z)}\right),
\end{dmath}
where the expectation is taken w.r.t. $Z \sim \mathcal{N}(0,1).$
\end{proposition}
Before providing a short sketch of the proof of the above proposition, we state a Lemma from \cite{hajek2015recovering}, which we need for our derivations.
\begin{lemma}\label{l:berryessen}\cite[Lemma ~11]{hajek2015recovering} Let $S_{\gamma} = X_1 + X_2 + \ldots + X_{N_{\gamma}},$ where $X_i,$ for $ i = 1, 2, \ldots N_{\gamma},$ are independent, identically distributed \gls{rv} with mean $\mu,$ variance $\sigma^2$ and $\mathbb E (|X_i^3|) \le g^3,$ and for some $\gamma > 0,$ $N_{\gamma}$ is a $\text{Poi}(\gamma)$ \gls{rv} independent of $ X_i: i = 1, 2, \ldots, N_{\gamma}.$ Then
\[
\textnormal{sup}_x  \left|\mathbb P \left( \frac{S_{\gamma} - \gamma \mu}{\sqrt{\gamma(\mu^2 + \sigma^2)}}\right) - \mathrm{\Phi}(x)\right| \le \frac{C_{BE}g^3}{\sqrt{\gamma(\mu^2 + \sigma^2)^3}},
\]
where $C_{BE} = 0.3041.$
\end{lemma}
We now provide a sketch of the proof of Proposition \ref{prop:largedegreeasyms}; the details can be found in Appendix \ref{ap:pr1}.\\
\begin{proof}[Sketch of Proof of Proposition~\ref{prop:largedegreeasyms}] The proof proceeds primarily by applying the expectation and variance operators to both sides of (\ref{eq:d1}) and (\ref{eq:d2}) and applying various reductions. First notice that when $a,b \to \infty$ and $\lambda$ and $\kappa$ are held constant, we have $\rho \to 1$ as follows:
\begin{equation}
\label{eq:rho}
\rho = a/b = 1 + \sqrt{\frac{\lal(1-\kappa)}{(1-\alpha)^2\kappa^2 b}}.
\end{equation}
Then using Taylor's expansion of $\log(1+x)$ we can expand the function $f(x)$ in (\ref{eq:fx}) up to second order as follows:
\begin{dmath}\label{eq:expfx}
f(x) = (\rho-1)\frac{e^x}{1+e^x} -\frac{1}{2} (\rho-1)^2 (\frac{e^x}{1+e^x})^2 + O(b^{-3/2}).
\end{dmath}
We use these expansions to simplify the expressions for the means and variances of (\ref{eq:d1}) and (\ref{eq:d2}). Then, by a change of measure, we express them in terms of functionals of a single rv, $\xi_1^t.$ We then use induction to show that the variance $\mu^{(t+1)}$ satisfies the recursion (\ref{eqrecu}) and use Lemma \ref{l:berryessen} to prove Gaussianity.
\end{proof}

In the following subsection, we use Proposition \ref{prop:largedegreeasyms} to derive the asymptotic error rates of the detector in Algorithm 1.

\subsection{Detection Performance}\label{sec:detmethod}

Let us use the symbol $\overline{S}$ to denote the subgraph nodes with the cued nodes removed, i.e., $\overline{S} = S \backslash C.$ This is the set that we aim to detect.
The output of Algorithm \ref{alg:bp_alg}, $\widehat S$ is the set of nodes with the top $K - |C|$ beliefs. We are interested in bounding the expected number of misclassified nodes $\mathbb E (|\overline S\Delta \widehat S|).$ Let $\widehat{S}$ be the output set of the algorithm excluding cues since the cues are always correctly detected. Note that $|\overline S| = |\widehat{S}| = K - |C|.$
To characterize the performance of the detector, we need to choose a performance measure. In \cite{montanari2015finding}, a rescaled probability of success was used to study the performance of a subgraph detector without cues, defined as
\begin{equation}\label{eq:defss}
P_{\textn{succ}}(\barsig) = \mathbb P (i \in \widehat{S}|i \in S) + \mathbb P (i \not\in \widehat{S}|i \not\in S) - 1,
\end{equation}
where $\barsig_i = \mathbf 1 (i \in \widehat S),$ and the dependence of $P_{\textn{succ}}(\barsig)$ on $n$ is implicit.
In our work, we study the following error measure, which is the average fraction of misclassified nodes, also considered in \cite{hajek2015recovering}, which for the uncued case is defined as
\[
\mathcal{E} := \frac{\mathbb E (|S\Delta \widehat{S}|)}{K}.
\]
Observe that $0\le\mathcal{E}\le 2.$ In particular $\mathcal{E}=2$ if the algorithm misclassifies all the subgraph nodes.
We now show that these two measures are roughly equivalent. For simplicity we consider the case where there are no cues, but the extension to the cued case is straightforward. Since our algorithm always outputs $K$ nodes as the subgraph, i.e., $|\widehat S| = K,$ the following is true for any estimate $\widehat{\sigma}$ of $\sigma:$
\begin{equation}\label{eq:errorrel}
r_n := \sum_{i=1}^n \mathbf{1}(\widehat{\sigma}_i = 0, i \in S) = \sum_{i=1}^n \mathbf{1}(\barsig_i = 1, i \not\in S),
\end{equation}
i.e., the number of misclassified subgraph nodes is equal to the number of misclassified nodes outside the subgraph.  We can rewrite the error measure $\mathcal E$ in terms of $r_n,$ since
\begin{align}
  \frac{\mathbb  |S\Delta \widehat{S}|}{K}  = \frac{2r_n}{K}\label{eq:defe}.
\end{align}
Next notice that we can rewrite $P_{\textn{succ}}(\barsig)$ as follows.
\begin{align}
  P_{\textn{succ}}(\barsig) &= 1 - \frac{1}{n}\sum_{i=1}^n \left(\mathbb P (\barsig_i = 0|i \in S) + \mathbb P (\barsig_i = 1|i \not\in S)\right)\nonum \\
  &\myeq{(a)} 1 - \sum_{i=1}^n \left( \frac{\mathbb P (\barsig_i = 0,i \in S)}{K} +  \frac{\mathbb P (\barsig_i = 1,i \not\in S)}{n-K}\right)\nonum \\
  &\myeq{(b)} 1 - \left( \frac{\mathbb E(r_n)}{K} +  \frac{\mathbb E(r_n)}{n-K}\right) = 1- \frac{n\mathbb E(r_n)}{K(n-K)},
\end{align}
where in step (a) we used Bayes rule with $\mathbb P(i\in S) = \frac{K}{n}.$ Since $1 \le \frac{n}{n-K} \le 2,$ we get
\begin{equation}\label{eq:defs}
   1 - 2\mathbb E(r_n)/K \le P_{\textn{succ}}(\barsig) \le 1 -  \mathbb E(r_n)/(K).
\end{equation}
Hence from (\ref{eq:defe}) and (\ref{eq:defs}),  $P_{\textn{succ}}(\barsig) \to 1$ if and only if $\frac{\mathbb E (|S\Delta \widehat{S}|)}{K} \to 0.$

In the following proposition, we state and prove the main result concerning the asymptotic error performance of Algorithm \ref{alg:bp_alg}.
\begin{theorem}\label{pr:weakrecovery}
For any $\lal > 0,\alpha > 0,$
\begin{equation}\label{eq:prweak}
 \lim_{b \to \infty} \lim_{n\to \infty} \frac{\mathbb{E} (|\overline S\Delta \widehat S|)}{K(1-\alpha)} \le 2\sqrt{\frac{1-\kappa}{\kappa(1-\alpha)}} e^{-\frac{1}{8} \frac{\alpha\lal(1-\kappa)}{\kappa(1-\alpha)^2}}.
\end{equation}
Consequently,
\[
\lim_{\kappa \to 0}\lim_{b \to \infty} \lim_{n\to \infty}  \frac{\mathbb{E} (|\overline S\Delta \widehat S|)}{K(1-\alpha)} = 0.
\]
\end{theorem}
\begin{proof}Let $\widehat{S}_0$ be the MAP estimator given by
\[
\widehat{S}_0 = \left\{i: R_i^t > \log\frac{1-\kappa }{\kappa(1-\alpha)} \right\}.
\]
Since $\widehat{S}$ is the set of nodes with the top $K - |C|$ beliefs, we have either $\widehat{S} \subset \widehat{S}_0 $ or $\widehat{S}_0 \subset \widehat{S}.$ Therefore,
\begin{align}
   |\overline S\Delta \widehat{S}| &\le |\overline S\Delta \widehat{S}_0| + |\widehat{S}\Delta \widehat{S}_0| \nonum \\
  &= |\overline S\Delta \widehat{S}_0| + |K - |C| - |\widehat{S}_0|| \nonum \\
  &= |\overline S\Delta \widehat{S}_0| + ||\overline S| - |\widehat{S}_0|| \nonum \\
  &\le 2|\overline S\Delta \widehat{S}_0|, \label{eq:boundexp}
\end{align}
where the last step follows because the set difference between two sets is lower bounded by the difference of their sizes. If we can bound $\frac{\mathbb{E} (|\overline S\Delta \widehat S_0|)}{K(1-\alpha)}$ by one-half the expression in (\ref{eq:prweak}) the result of the Proposition follows. The proof of this upper bound uses Proposition \ref{prop:largedegreeasyms} and is given in Appendix \ref{appr2}.
\end{proof}
Theorem \ref{pr:weakrecovery} states that the detectability threshold does not exist for Belief Propagation with cues.

This is in stark contrast to the performance of BP when there is no side-information. In that case, as stated in the following theorem from \cite{montanari2015finding}, the performance of any local algorithm suffers when the SNR parameter $\lambda < 1/e.$ In the following $\textnormal{LOC}$ denotes the class of all local algorithms, i.e., algorithms that take as input the local neighbourhood of a node.
\begin{proposition}\cite[Theorem~1]{montanari2015finding}\label{thm:monta}
If $\lambda < 1/e,$ then all local algorithms have success probability uniformly bounded away from one; in particular,
\[
\sup_{T \in \textnormal{LOC}}\lim_{n \to \infty} P_{\textnormal{succ}}(T) \le \frac{e-1}{4},
\]
and therefore
\[
\sup_{T \in \textnormal{LOC}}\lim_{n \to \infty} \mathcal E(T) \ge \frac{5-e}{4} > 1/2.
\]
\end{proposition}

\section{Imperfect Side Information}\label{impsibp}
In this section, we develop a BP algorithm under the more realistic assumption of imperfect side information, where the available cue information is not completely reliable. This is true of humanly classfied data available for many semi-supervised learning problems.


Our BP algorithm can easily take into account imperfection in side information. Suppose we know the parameters $\alpha$ and $\beta$ defined in (\ref{eq:alpha}) and (\ref{def:beta}) respectively, or their estimates thereof. We remark that unlike Algorithm \ref{alg:bp_alg}, which only has to detect the uncued subgraph nodes, our algorithm needs to explore the whole graph, since we do not know a priori which cues are correct.
As before, for a node $u,$ we wish to compute the following log-likelihood ratio in a distributed manner:
\[
R_u^t = \log\left(\frac{\mathbb P (G_u^t,c_u,C_u^t|\sigma_u=1)}{\mathbb P(G_u^t,c_u,C_u^t|\sigma_u=0)}\right),
\]
where $c_u$ is the indicator variable of whether $u$ is a cued node, and $C_u^t$ is the cued information of the $t$-hop neighbourhood of $u,$ excluding $u.$ Note that we can expand $R_u^t$ as follows
\begin{align}
R_u^t &= \log\left(\frac{\mathbb P (G_u^t,C_u^t|\sigma_u=1,c_u)}{\mathbb P(G_u^t,C_u^t|\sigma_u=0,c_u)}\right) + \log \left(\frac{\mathbb P(c_u|\sigma_u = 1)}{\mathbb P(c_u|\sigma_u = 0)}\right)\nonumber \\
&= \log\left(\frac{\mathbb P (G_u^t,C_u^t|\sigma_u=1)}{\mathbb P(G_u^t,C_u^t|\sigma_u=0)}\right) + \log \left(\frac{\mathbb P(c_u|\sigma_u = 1)}{\mathbb P(c_u|\sigma_u = 0)}\right),\label{eq:step2}
\end{align}
where in the second step we dropped the conditioning w.r.t. $c_u$ because $(G_u^t,C_u^t)$ is independent  of the cue information of node $u$ given $\sigma_u.$
Let $h_u = \log \left(\frac{\mathbb P(c_u|\sigma_u = 1)}{\mathbb P(c_u|\sigma_u = 0)}\right).$ Then it is easy to see from (\ref{eq:samimp}) that
\begin{equation}\label{eq:def_h}
h_u =
\begin{cases}
  \log\left(\frac{\beta(1-\kappa)}{(1-\beta)\kappa}\right), & \text{ if } u \in C, \\
  \log\left(\frac{(1-\alpha\beta)(1-\kappa)}{(1-\kappa-\alpha\kappa+\alpha\kappa\beta)}\right), &\text{ otherwise}.
\end{cases}
\end{equation}
The recursion for the first term in (\ref{eq:step2}) can be derived along the same lines as the derivation of Algorithm \ref{alg:bp_alg} and is skipped. The final BP recursions are given in Algorithm \ref{alg:bp_alg2}.
\begin{algorithm}
  \caption{BP with imperfect cues}
  \label{alg:bp_alg2}
  \begin{algorithmic}[1]
    \STATE Initialize: Set $R^0_{i \to j}$ to 0, for all $(i,j) \in E.$ Let $t_f < \frac{\log(n)}{\log(np)}+1.$ Set $t=0.$
    \STATE For all directed pairs $(i,u) \in E$:
    \begin{dmath}
\label{eq:bpupdate2}
    R_{i \to u}^{t+1} = -K(p-q) + h_i +  \sum_{l \in \delta i,l \neq  u} \log \left(\frac{\exp(R^t_{l \to i}-\nu) (p/q) +1}{\exp(R_{l \to i}^t - \nu) + 1}\right),
\end{dmath}
where $\nu = \log(\frac{n-K}{K}).$
    \label{op1}
    \STATE Increment t; if $t< t_f-1$ go back to \ref{op1}, else go to \ref{op2}
    \STATE Compute $R_u^{t_f}$ for every $u \in V$ as follows: \label{op2}
      \begin{dmath}
\label{eq:bpbelief2}
    R_u^{t+1} = -K(p-q) + h_u + \sum_{l \in \delta u} \log \left(\frac{\exp(R^t_{l \to u}-\nu) (p/q) +1}{\exp(R_{l \to u}^t - \nu)+ 1}\right)
\end{dmath}
    \STATE Output $\widehat S$ as $K$ set of nodes in $V$ with the largest values of $R_u^{t_f}.$

      \end{algorithmic}
\end{algorithm}

In order to analyze the error performance of this algorithm we derive the asymptotic distributions of the messages $R_{u\to i}^t,$ for  $\{\sigma_u = 0\}$ and $\{\sigma_u = 1\}$. Note that, since we now assume that we do not know the exact classification of any of the subgraph nodes, we need to detect $K$ nodes, and hence the effective SNR parameter is defined as
\begin{equation}\label{eq:def_lam2}
\lambda = \frac{K^2(p-q)^2}{(n-K)q}.
\end{equation}
The following proposition presents the asymptotic distribution of the messages $R_{u\to i}^t$ in the limit of $n \to \infty$ and in the large degree regime where $a,b \to \infty.$
\begin{proposition}\label{l:bpdistr2}
Let $n \to \infty.$ In the regime where $\lambda$ and $\kappa$ are held fixed and $a,b \to \infty,$ the message $R^t_{u\to i}$ given $\{\sigma_u = j\},$ where $j = \{0,1\}$ converges in distribution to $\Gm_j^t + h_u$  where $h_u$ is defined in (\ref{eq:def_h}). The \gls{rv}s $\Gm_j^t$ have the following distribution:
  \begin{align*}
  \Gm_0^t &\sim \mathcal{N}(-\mu^{(t)}/2,{\mu^{(t)}}), \textnormal{and}\\
  \Gm_1^t &\sim \mathcal{N}( \mu^{(t)}/2,{\mu^{(t)}}),
\end{align*}
where $\mu^{(t)}$ satisfies the following recursion with $\mu^{(0)} = 0,$
  \begin{dmath}\label{eq:recmu2}
    \mu^{(t+1)} = \alpha\beta^2 \lambda \mathbb E \left(\frac{(1-\kappa)/\kappa }{\beta + (1-\beta)e^{(-\mu^{(t)}/2-\sqrt{\mu^{(t)}}Z)}}\right) + (1-\alpha\beta)^2\lambda\\ \mathbb E \left(\frac{(1-\kappa)}{\kappa(1-\alpha\beta)+(1-\kappa -\alpha\kappa +\alpha\kappa\beta)e^{(-\mu^{(t)}/2-\sqrt{\mu^{(t)}}Z)}}\right),
  \end{dmath}
and the expectation is \gls{wrt} $Z \sim \mathcal N(0,1).$
\end{proposition}
\begin{proof}The proof proceeds by deriving the recursive distributional equations that the message distributions satisfy in the limit $n \to \infty,$ and then applying the large degree limit of $a,b \to \infty$ to these recursions.The details are in the supplementary material.\end{proof}

\noindent The above proposition immediately leads to the following result on the asymptotic error rate of Algorithm \ref{alg:bp_alg2}.

\begin{theorem}\label{thm:bpimpsideinfo}
For any $\lambda>0,\alpha>0,\beta > 0,$
  \begin{eqnarray*}
\lefteqn{\lim_{b \to \infty}\lim_{n \to \infty} \frac{\mathbb E (|\hat{S} \Delta S|)}{K}}\\
 &\le& 2\biggl(\alpha\sqrt{\beta(1-\beta)} + \\
 &&\sqrt{(1-\alpha\beta) (\frac{1-\kappa}{\kappa} - \alpha(1-\beta))}\biggr) e^{-\frac{\lambda \alpha \beta^2 (1-\kappa)}{8\kappa}}.
  \end{eqnarray*}
Consequently,
\[
\lim_{\kappa \to 0} \lim_{b\to\infty} \lim_{n\to \infty} \frac{\mathbb E(|S\Delta\widehat{S}|)}{K} = 0.
\]
\end{theorem}
\begin{proof}The proof essentially analyzes the properties of the recursion (\ref{eq:recmu2}) and is similar to the proof of Theorem \ref{pr:weakrecovery}. See supplementary material for details.\end{proof}

\section{Numerical Experiments} \label{sec:sims}

In this section we provide numerical results to validate our theoretical findings on the synthetic model as well as on two real-world datasets. We compare the performance of BP to another seed-based community detection algorithm, the personalized PageRank, which is widely used for local community detection \cite{Andersen2007}.
\subsection{Synthetic dataset}
First we show that the limitation of local algorithms described in Proposition \ref{thm:monta} is overcome by BP when there is non-trivial side-information. Proposition \ref{thm:monta} says that when $\lambda < 1/e,$ $\mathcal E(T) > 1/2$ for any local algorithm $T.$ We run our Algorithm \ref{alg:bp_alg}, on a graph generated with  $\alpha = 0.1, \kappa = 5\times10^{-4},b = 100$ and $n = 10^6.$ For $\lambda = 1/4 < 1/e,$ we get an average value of $\mathcal E = 0.228 < 1/2.$ Thus it is clear that our algorithm overcomes the computational threshold of $\lambda = 1/e.$

Next, we study the performance of Algorithm \ref{alg:bp_alg2} when there is noisy side-information with $\beta = 0.8.$ For $\lambda = 1/3 < 1/e,$ we get an average error rate of $0.3916 < 1/2$ clearly beating the threshold of $\lambda = 1/e.$ Thus we have demonstrated that both with perfect and imperfect side-information, our algorithm overcomes the $\lambda = 1/e$ barrier of local algorithms.

Next, we verify that increasing $\alpha$ improves the performance of our algorithm as expected. In Figure \ref{fig:figurealpha}, we plot the variation of $\mathcal E$ of Algorithm \ref{alg:bp_alg} as a function of $\alpha.$ Our parameter setting is $\kappa = 0.01, b = 100,$ and $\lambda = 1/2$ with $n = 10^4.$ In the figure, we also plot the error rate $\mathcal E$ obtained by personalized PageRank under the same setting, with damping factor $\alpha_{pr} = 0.9$ \cite{Andersen2007}. The figure demonstrates that BP benefits more as the amount of side-information is increased than PageRank does.

Next, we compare the performance of BP algorithm without side-information given in \cite{montanari2015finding} to our algorithm with varying amounts of side-information. We choose the setting where $n = 10^4, b = 140$ and $\kappa = 0.033$ for different values of $\lambda$ by varying $p.$ In Figure \ref{fig:figure2} we plot the metric $\mathcal E$ against $\lambda$ for different values of $\beta,$ with $\alpha = 0.1$. For $\beta = 1$ we use Algorithm \ref{alg:bp_alg}. We can see that even BP with noisy side-information performs better than standard BP with no side-information. In addition, as expected increasing $\beta$ improves the error performance.
\begin{figure}
\centering
\includegraphics[scale=0.5]{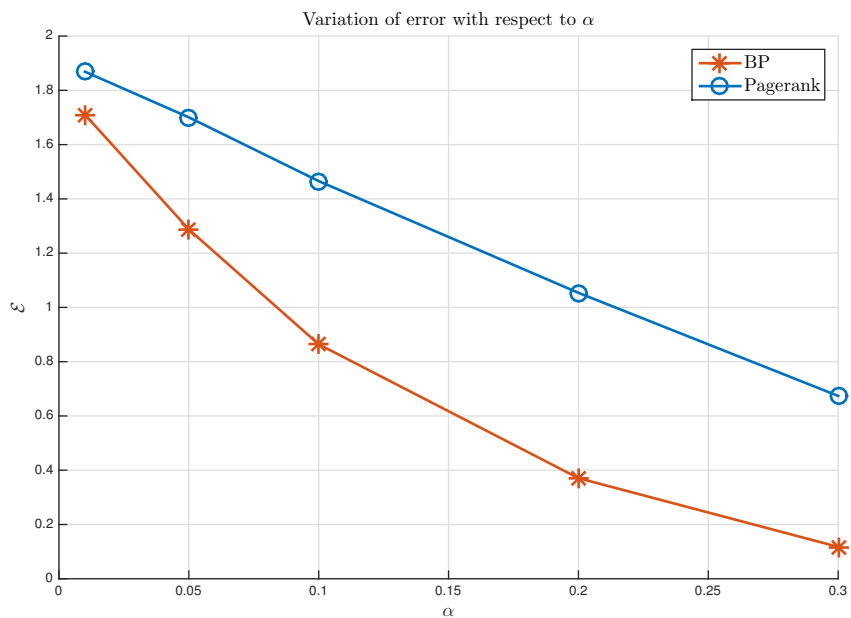}
\caption{Performance of BP Algo \ref{alg:bp_alg} as a function of $\alpha$}
\label{fig:figurealpha}
\end{figure}
\begin{figure}
\centering
\includegraphics[scale=0.5]{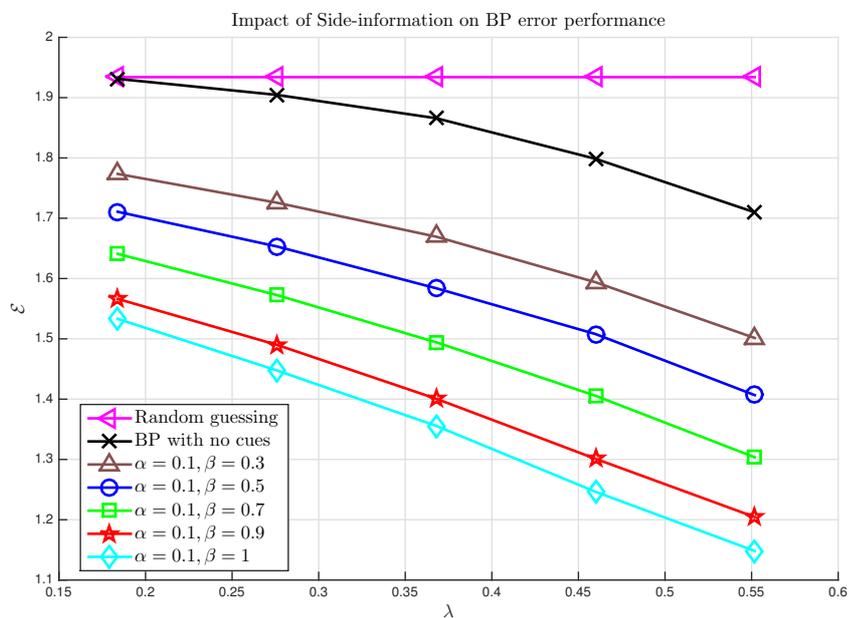}
\caption{Comparison of BP for subgraph detection for different amounts of side-information}
\label{fig:figure2}
\end{figure}
\subsection{Real-world datasets}
We consider two real-world networks: The USPS dataset and the Reuters-911 dataset. For these two datasets we compare the performance of BP with personalized PageRank in terms of recall rate $\mathcal{R}$ defined as
\[
\mathcal{R} = \frac{|S\cap\hat{S}|}{|S|},
\]
where $S$ is the true community and $\hat{S}$ is its estimate. This is a commonly used metric for community detection applications \cite{yang2015defining}. We use $\alpha_{\textnormal{pr}} = 0.9$ as the damping factor of PageRank.
We describe the datasets and the results obtained by our algorithms below.
\subsubsection{USPS dataset}
The USPS dataset contains 9296 scanned images of size 16 $\times$ 16, which can represented by a feature vector of size $256\times 1$ with values from -1 to +1 \cite{Zhou2004}. First, we construct a graph from this dataset, where nodes represent scanned images, by adding a link between a node and its three nearest neighbours, where the distance is defined as the euclidean distance between the images represented as feature vectors. The resulting graph is undirected with a minimum degree of at least 3. This is an instance of the $k$ nearest neighbour graph, with $k=3.$ On this graph we run BP and PageRank separately for each of the 10 communities for $\alpha = 0.01$ and $\alpha = 0.05$ (Figure~\ref{fig:bpprankusps}). It can be seen from Figure \ref{fig:bpprankusps}, that the performance of BP is strictly worse than that of PageRank. This result points to the importance of having the correct initialization for the BP parameters. Indeed, in our underlying model for BP, we assumed that there is only one dense community in a sparse network, in which case, as demonstrated in Figure\ref{fig:figurealpha}, BP outperforms PageRank by a big margin. However in the USPS graph, there are ten dense communities, and therefore it deviates significantly from our underlying model.
\begin{figure}
\centering
\includegraphics[scale=0.5]{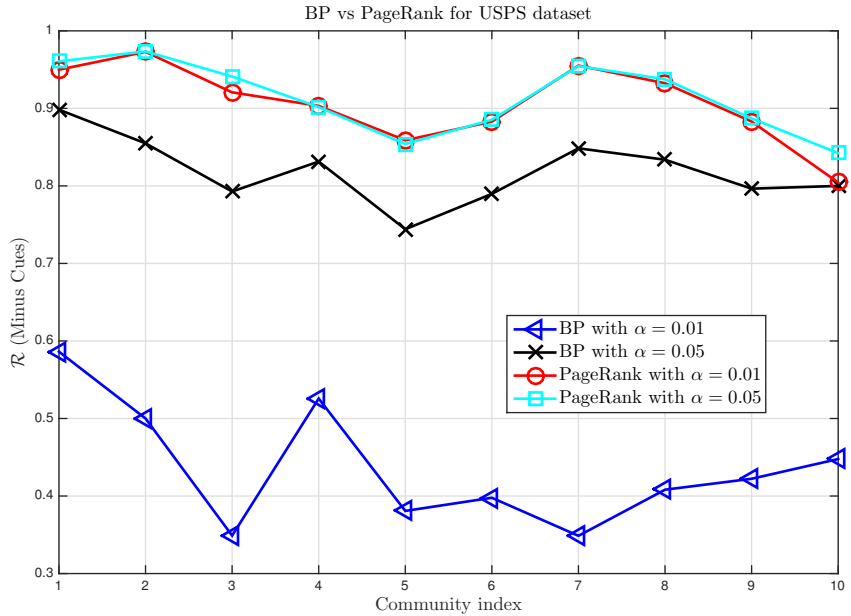}
\caption{Comparison of BP for subgraph detection for different amounts of side-information}
\label{fig:bpprankusps}
\end{figure}
\subsubsection{Reuters911 Dataset}
In this subsection we consider a graph that is closer to our assumed model. We consider the Reuters911 dataset also used in \cite{chen2012dense}. It is made up of words from all news released by Reuters for 66 days since September 11, 2001. Table~5 in \cite{chen2012dense} shows a group of 99 collocated words in this dataset. This subset represents the largest dense community to be detected in this dataset. A graph of size $n = 13332$ is generated from this dataset by adding a link between two words if they appear together in a sentence. The resulting graph is undirected and unweighted. We compare BP and Pagerank on this dataset for one and two cues. The cues we use are the words \textit{pentagon} and \textit{11}. In Table \ref{tabl2} we show the recall values $\mathcal R$ of PageRank and BP, excluding cues. Clearly, BP performs better.
\begin{table}
    \begin{tabular}{ | l | l | l| p{1.4cm} |}
    \hline
    Class $0$ & \begin{tabular}{l} $\# \textnormal{of cues} = 1$\end{tabular} & \begin{tabular}{l} $\# \textnormal{of cues} = 2$\end{tabular}  \\ \hline
    BP & \textbf{0.7143} & \textbf{0.7216} \\ \hline
    PageRank & 0.6327 &  0.6392 \\ \hline
    \end{tabular}
    \caption{Reuters911 recall results}
    \label{tabl2}
\end{table}

\section{Conclusions and Future Extensions}\label{sec:con}
In this work we developed a local distributed  BP algorithm that takes advantage of side-information to detect a dense subgraph embedded in a sparse graph. We obtained theoretical results based on density evolution on trees to show that it achieves zero asymptotic error  regardless of the SNR parameter $\lambda,$ unlike BP without cues, where there is a non-zero detectability threshold. We then validated our theoretical results by simulating our algorithm on a synthetic dataset and showing that, in the presence of both noise-less and noisy side-information, our BP algorithm overcomes the error bound of local algorithms when $\lambda < 1/e.$ We then applied our algorithm to two real-world datasets: USPS and Reuters911 and compared its performance with personalized PageRank. Our results indicate that the relative improvement in BP depends on the closeness of the dataset to the underlying graph model used to derive BP. In the future, we would like to do non-asymptotic analysis when $a, b$ and $\kappa$ are functions of $n.$ Extension to dense graphs would also be interesting, where traditional BP and tree coupling-based analysis will not work owing to the presence of loops.

\section*{Acknowledgements}
This work was partly funded by the French Government (National Research Agency, ANR) through the ``Investments for the Future'' Program reference \#ANR-11-LABX-0031-01 and Indo-French CEFIPRA Collaboration
Grant No.5100-IT1 ``Monte Carlo and Learning Schemes for Network Analytics.''

\appendix

\section{Description of G-W tree and derivation of Algorithm \ref{alg:bp_alg}}\label{ap:gwtree}
We derive Algorithm \ref{alg:bp_alg} by establishing a coupling formulation between a $t$-hop neighbourhood $G_u^t$ of node $u$ and a Galton-Watson (G-W) tree rooted at $u$ constructed as follows.
Let $T_u^t$ be a labelled Galton-Watson (G-W) tree of depth $t$ rooted at node $u$ constructed as follows (as in \cite{hajek2015recovering}): The label $\tau_u$ at node $u$ is chosen at random in the following way:
\begin{align*}
\mathbb{P} ( \tau_u = 1 ) = \frac{K}{n}, && \mathbb{P} ( \tau_u = 0) = \frac{n-K}{n}.
\end{align*}
The number of children $N_u$ of the root $u$ is Poisson-distributed with mean $d_1 = Kp + (n-K)q$ if $\tau_u = 1$ and mean $d_0 = n q$ if $\tau_u = 0.$ Each child is also assigned  a label. The number of children $i$ with label $\tau_i = 1$ is Poisson distributed with mean $Kp$ if $\tau_u = 1$ and mean $Kq$ if $\tau_i = 0.$ The number of children with label $\tau_i = 0$ is Poisson distributed with mean $(n-K)q$ for both $\tau_u = 0$ and $\tau_u = 1.$ By the independent splitting property of Poisson random variables, this is equivalent to assigning the label $\tau_i = 1$ to each child $i$ by sampling a Bernoulli random variable with probability (w.p.) $Kp/d_1$ if $\tau_u = 1$ and $Kq/d_0$ if $\tau_u = 0.$ Similarly $\tau_i = 0$ w.p. $(n-K)q/d_1$ and $(n-K)q/d_0$ for $\tau_u = 0$ and $1$ respectively. Namely, if $i$ is a child of $u,$
\begin{align}
\label{eq:condprobs}
\mathbb{P}(\tau_i = 1 |\tau_u = 1) = \frac{Kp}{d_1}, && \mathbb{P}(\tau_i = 1 | \tau_u = 0) = \frac{Kq}{d_0}.
\end{align}

We then assign the cue indicator function $\widetilde c$ such that $\widetilde c_i = 1$ w.p. $\alpha$ if $\tau_i = 1$ and $\widetilde c_i = 0$ if $\tau_i = 0.$ The process is repeated up to depth $t$ giving us $\widetilde C_u^t,$ the set of cued neighbours.
Now we have the following coupling result between $(G_u^t,\sigma^t,C_u^t),$ the neighbourhood of $u$ and the node labels of that neighbourhood and $(T_u^t,\tau^t,\widetilde C_u^t),$ the depth-$t$ tree $T_u^t$ and its labels due to \cite{hajek2015recovering}.
\begin{lemma}\cite[Lemma~15]{hajek2015recovering}
\label{l:coupling}
For $t$ such that $(np)^t = n^{o(1)},$ there exists a coupling such that $(G_u^t,\sigma^t,C_u^t) = (T_u^t,\tau^t,\widetilde C_u^t)$  with probability $1 - n^{-1 + o(1)}$.
\end{lemma}

We now derive the recursions for the likelihood ratios on the tree $T_u^t.$ For large $n$ with high probability, by the coupling formulation, $R_u^t$ also satisfy the same recursions. For notational simplicity, from here onwards we represent the cue labels on the tree by $c$ and the set of cued neighbours by $C_u^t,$ just as for the original graph. We use $\Lambda_u^t$ to denote the likelihood ratio of node $u$ computed on a tree defined as below:
\[
\Lambda_u^{t+1} = \log \left(\frac{\mathbb P (T_u^{t+1},C_u^{t+1}|\tau_u = 1)}{\mathbb P (T_u^{t+1},C_u^{t+1}|\tau_u = 0) }\right).
\]
By virtue of tree construction, if the node $u$ has $N_u$ children, the $N_u$ subtrees rooted on these children are jointly independent given $\tau_u.$ We use this fact to split $\Lambda_u^{t+1}$ in two parts.

\begin{align}
\Lambda_u^{t+1} &= \log \left(\frac{\mathbb P (T_u^{t+1},C_u^{t+1}|\tau_u = 1)}{\mathbb P (T_u^{t+1},C_u^{t+1}|\tau_u = 0) }\right) \nonumber \\
&= \log\left({\frac{\mathbb{P}(N_u|\tau_u =1 )}{\mathbb{P} (N_u |\tau_u = 0)}}\right)  +\\& \sum_{i \in \delta u} \log \left(\frac{\mathbb{P} (T_i^{t},c_i,C_i^t|\tau_u = 1)}{\mathbb{P} (T_i^{t},c_i,C_i^t|\tau_u = 0) }\right), \label{eq:lambdat}
\end{align}
by the independence property of subtress $T_i^t$ rooted on $i \in \delta u$. Since by Lemma \ref{l:coupling}, the degrees are Poisson,
\[
\mathbb{P}(N_u|\tau_u =1) = d_1^{N_u} e^{-d_1}/N_u!,
\]
and similarly for $\mathbb{P}(N_u|\tau_u =0).$ Therefore we have
\begin{align}
\log\left(\frac{\mathbb{P}(N_u|\tau_u =1)}{\mathbb{P} (N_u |\tau_u = 0)} \right) &= N_u \log\left(\frac{d_1}{d_0}\right) - (d_1 - d_0) \nonumber \\
&= N_u \log\left( \frac{d_1}{d_0} \right) -K(p-q).\label{eq:term1}
\end{align}
Next we look at the second term in (\ref{eq:lambdat}). We analyze separately the case of $c_i = 1$ and $c_i = 0$ for $i \in \delta_u,$ i.e, the cued and uncued children are handled separately.

\textit{Case 1} ( $c_i = 1$): We have

\begin{eqnarray}
\lefteqn{\log\left(\dfrac{\mathbb{P} (T_i^{t},c_i,C_i^t|\tau_u = 1)}{\mathbb{P} (T_i^{t},c_i,C_i^t|\tau_u = 0) }\right)}\nonumber \\
&\myeq{(a)}& \log\left( \dfrac{\mathbb{P} (T_i^{t},c_i,C_i^t,\tau_i = 1|\tau_u = 1)}{\mathbb{P} (T_i^{t},c_i,C_i^t,\tau_i = 1|\tau_u = 0)} \right)\nonumber \\
&=& \log\left( \dfrac{\mathbb{P} (T_i^{t},c_i,C_i^t|\tau_i = 1) \mathbb{P}(\tau_i = 1|\tau_u = 1)}{\mathbb{P} (T_i^{t},c_i,C_i^t|\tau_i = 1) \mathbb{P}(\tau_i = 1|\tau_u = 0)} \right)\nonumber \\
&\myeq{(b)}& \log \left(\dfrac{Kp/d_1}{Kq/d_0}\right), \label{eq:term2}
\end{eqnarray}
where in step (a) we applied the fact that $c_i=1$ implies $\tau_i = 1,$ and in (b) we used (\ref{eq:condprobs}).

\textit{Case 2} ($c_i = 0$): Observe that $\mathbb{P}(c_i = 0|\tau_i = 1) = 1 - \alpha$ and $\mathbb{P}(c_i = 0|\tau_i = 0) = 1.$ Note that

\begin{eqnarray}
\lefteqn{\mathbb{P} (T_i^{t},c_i,C_i^t|\tau_u = 1)}\nonumber \\
&&= \quad \mathbb{P} (T_i^{t},C_i^t|\tau_i = 1) \mathbb{P}(c_i|\tau_i = 1) \mathbb{P}(\tau_i= 1|\tau_u = 1)  \nonumber\\
&&\quad + \quad \mathbb{P} (T_i^{t},C_i^t|\tau_i = 0) \mathbb{P}(c_i|\tau_i = 0) \mathbb{P}(\tau_i= 0|\tau_u = 1) \nonumber\\
&&= \quad \mathbb{P} (T_i^{t},C_i^t|\tau_i = 1) (1 - \alpha) \frac{Kp}{d_1}  \nonumber\\
&& \quad+ \quad \mathbb{P} (T_i^{t},C_i^t|\tau_i = 0) \frac{(n-K)q}{d_1}.\label{eq:T5}
\end{eqnarray}

Similarly, we can show
\begin{align}
\mathbb{P} (T_i^{t},c_i,C_i^t|\tau_u = 0) &= \mathbb{P} (T_i^{t},C_i^t|\tau_i = 1) \frac{Kq}{d_0} (1 - \alpha) \nonumber \\
 &\quad+ \quad  \mathbb{P} (T_i^{t},C_i^t|\tau_i = 0) \frac{(n-K)q}{d_0}\label{eq:T4}.
\end{align}
Let us define
\[\Lambda_{i \to u}^t := \log \left(\frac{\mathbb{P} (T_i^{t},C_i^t|\tau_i = 1)}{\mathbb{P} (T_i^{t},C_i^t|\tau_i = 0)}\right),\]
the message that $i$ sends to $u$ at step $t$.
Using the above definition, (\ref{eq:T5}), and (\ref{eq:T4}) we get
\begin{eqnarray}
\lefteqn{\log \left(\frac{\mathbb{P} (T_i^{t},c_i,C_i^t|\tau_u = 1)}{\mathbb{P} (T_i^{t},c_i,C_i^t|\tau_u = 0) }\right)} \nonum\\
&=&\log \left(\frac{e^{\Lambda_{i\to u}^t} \frac{Kp}{d_1}(1-\alpha)+\frac{(n-K)q}{d_1}}{e^{\Lambda_{i\to u}^t}\frac{Kq}{d_0}(1-\alpha)+\frac{(n-K)q}{d_0}}\right)\nonum \\
&=& \log\left(\frac{d_0}{d_1}\right) + \log \left(\frac{e^{\Lambda_{i\to u}^t} \frac{Kp}{(n-K)q}(1-\alpha)+1}{e^{\Lambda_{i\to u}^t}\frac{K}{(n-K)}(1-\alpha)+1}\right).\label{eq:term3}
\end{eqnarray}
We then use the substitution $\nu := \log((n-K)/K)$ in the above equation.
 Finally combining (\ref{eq:term1}), (\ref{eq:term2}) and (\ref{eq:term3}) and replacing $\Lambda_u^t$ with $R_u^t$ and $\Lambda_{i \to u}^t$ with $R_{i \to u}^t$, we arrive at (\ref{eq:bpbelief}). The recursive equation (\ref{eq:bpupdate}) can be derived in exactly the same way by looking at the children of $i \in \delta u.$

\section{Proof of Proposition \ref{prop:largedegreeasyms}}\label{ap:pr1}
Since the statistical properties of $R_u^t$ and $\Lambda_u^t$ are the same in the $n \to \infty$ limit, we analyze the distribution of $\Lambda_u^t.$ Let us define the posterior likelihood for $\tau_u$ given by
\[
\widetilde{\Lambda}_i^t = \log\left(\frac{\mathbb{P}(\tau_i =1|T_i^t,C_i^t,c_i=0)}{\mathbb{P}(\tau_i =0|T_i^t,C_i^t,c_i=0)}\right).
\]
Note that $\mathbb P (\tau_i = 1 | c_i = 0) = \kappa(1-\alpha)/(1-\kappa\alpha)$ and $\mathbb P (\tau_i = 0|c_i = 0) = (1-\kappa)/(1-\kappa\alpha)$ are the prior probabilities of the uncued vertices. For convenience we use an overline for the symbols of expectation $\mathbbop E$ and probability $\mathbbop P$ to denote conditioning w.r.t $\{c_i=0\}.$

By a slight abuse of notation, let $\xi_0^t$ and $\xi_1^t$ denote the rvs whose distributions are the same as the distributions of $\widetilde{\Lambda}_i^t$ given $\{c_i=0,\tau_i =0\}$ and $\{c_i=0,\tau_i = 1\}$ respectively in the limit $n \to \infty.$ We need a relationship between $P_0$ and $P_1,$  the probability measures of $\xi_0^t$ and $\xi_1^t$ respectively, stated in the following lemma.
\begin{lemma}\label{l:changmea}
\[
\frac{dP_0}{dP_1}(\xi) = \frac{\kappa (1-\alpha)}{1-\kappa} \exp(-\xi).
\]
In other words for any integrable function $g(\cdot)$
\[
\mathbbop{E}[g(\widetilde{\Lambda}^t_u) |\tau_u = 0] = \frac{\kappa (1-\alpha)}{1-\kappa} \mathbbop{E}[g(\widetilde{\Lambda}^t_u) e^{-\widetilde{\Lambda}^t_u}|\tau_u = 1].
\]
\end{lemma}
\begin{proof} Following the logic in \cite{montanari2015finding}, we show this result for $g(\widetilde\Lambda^t_u) = \mathbf{1}(\widetilde\Lambda_u \in A),A$ being some measurable set. The result for general $g$ then follows because any integrable function can be obtained as the limit of a sequence of such rvs \cite{Bilsley}. Let $Y = (T_u^t,C_u^t),$ the observed \gls{rv}. Therefore
\begin{eqnarray*}
\lefteqn{\mathbbop{E} \left(\mathbf 1\left(\widetilde\Lambda^t_u \in A\right) |\tau_u = 0\right)}\\
 &=& \mathbbop{P}\left(\widetilde\Lambda^t_u \in A|\tau_u = 0\right) \\
&=& \frac{\mathbbop P (\widetilde\Lambda^t_u \in A, \tau_u = 0)}{\mathbbop{P}(\tau_u=0)} \\
&=& \frac{\mathbbop{E}_Y \left(\mathbbop{P}(\widetilde\Lambda^t_u \in A, \tau_u = 0| Y)\right)}{\mathbbop{P}(\tau_u=0)}\\
&=& \mathbbop{E}_Y\left[\frac{\mathbf{1}(\widetilde\Lambda^t_u \in A) \mathbbop{P}({\tau_u= 0|Y})}{\mathbbop{P}(\tau_u = 0)}\right]\\
&\myeq{(a)}&  \mathbbop{E}_Y\left(\frac{\mathbf{1}(\widetilde\Lambda^t_u \in A)e^{-{\widetilde\Lambda^t_u}} \mathbbop{P}({\tau_u = 1|Y})}{\mathbbop{P}(\tau_u = 0)}\right)\\
&=& \frac{\mathbbop P (\tau_u = 1)}{\mathbbop P (\tau_u = 0)} \mathbbop{E}_1(\mathbf{1}(\widetilde\Lambda^t_u \in A) e^{-{\widetilde\Lambda^t_u}})\\
&=& \frac{\kappa (1-\alpha)}{1 -\kappa}\mathbbop{E}_1(\mathbf{1}(\widetilde\Lambda^t_u \in A) e^{-{\widetilde\Lambda^t_u}}),
\end{eqnarray*}
where in (a) we used the fact that $\frac{\mathbbop P (\tau_u = 0|Y)}{\mathbbop P (\tau_u = 1|Y)} = \exp(-\widetilde\Lambda^t_u),$ and $\mathbb E_1$ denotes expectation conditioned on the event $\{\tau_u = 1\}.$
\end{proof}

\begin{proof}

Since $\lal$ and $\kappa$ are fixed and $b \to \infty,$ from (\ref{eq:effsnr}) we have
\begin{equation}
\label{eq:rho}
\rho := a/b = 1 + \sqrt{\frac{\lal(1-\kappa)}{(1-\alpha)^2\kappa^2 b}} = 1 + O(b^{-1/2}).
\end{equation}

Following \cite{montanari2015finding}, we prove the result by induction on $t$. First let us verify the result holds when $t=0,$ for the initial condition that $\xi_0^0 = \xi_1^0 = -\upsilon.$ We only do this for $\xi_0^{t}$ since the steps are similar for $\xi_1^t.$ Observe that

\begin{align}
f(-\upsilon)  &= \log\left( \frac{\frac{\kappa(1-\alpha)\rho}{(1-\kappa)} +1 }{ \frac{\kappa(1-\alpha)}{(1-\kappa)} +1 }\right)\nonumber \\
&= \log\left(  1 + (\rho-1) \frac{\kappa (1-\alpha)}{1-\kappa\alpha} \right)\nonumber \\
&\myeq{(a)}(\rho-1) \frac{\kappa (1-\alpha)}{1-\kappa\alpha} - \frac{(\rho-1)^2}{2} \frac{\kappa^2 (1-\alpha)^2}{(1-\kappa\alpha)^2} \nonum \\&+ O(b^{-3/2}),\label{eq:fups}
\end{align}
where $(a)$ follows from (\ref{eq:rho}), and Taylor's expansion around $\rho = 1.$ Similarly,
\begin{equation}\label{eq:fups2}
f^2(-\upsilon) = (\rho-1)^2\frac{\kappa^2(1-\alpha)^2}{(1 - \kappa\alpha)^2} + O(b^{-3/2}),
\end{equation}
\begin{align}
\log(\rho) &= \log(1 + (\rho-1)) \nonumber \\ &= \sqrt{\frac{\lal(1-\kappa)}{(1-\alpha)^2\kappa^2b}} - \frac{\lal(1-\kappa)}{2(1-\alpha)^2\kappa^2 b} + O(b^{-3/2})\label{eq:logrho},
\end{align}
and
\begin{equation}\label{eq:logrho2}
\log^2(\rho) = \frac{\lal(1-\kappa)}{(1-\alpha)^2\kappa^2 b} + O(b^{-3/2}).
\end{equation}
Let us verify the induction result for $t=0.$ Using the recursion (\ref{eq:d1}) with $\xi_0^{0} = \log\frac{\kappa(1-\alpha)}{1-\kappa} = -\upsilon,$ we can express $\mathbb{E}\xi_0^{1}$ as
\begin{equation*}
\mathbb{E}\xi_0^{1} =
-\kappa b (\rho-1) - \upsilon + \kappa b \alpha \log(\rho)
 + b(1-\kappa\alpha) f(-\upsilon).
\end{equation*}
Now using (\ref{eq:fups}) and (\ref{eq:logrho}) we obtain
\begin{align}
\mathbb{E}\xi_0^{1} &= -\kappa\sqrt{\frac{\lal b (1-\kappa)}{(1-\alpha)^2\kappa^2}} - \upsilon + \kappa \alpha \sqrt{\frac{\lal(1-\kappa)b}{(1-\alpha)^2\kappa^2}} \\ &-\frac{\lal(1-\kappa)\alpha}{2(1-\alpha)^2\kappa} \nonum \\
& + \sqrt{\frac{\lal(1-\kappa)b}{(1-\alpha)^2\kappa^2}}\kappa(1-\alpha) - \frac{\lal(1-\kappa)}{2(1-\kappa\alpha)}+ O(b^{-1/2}) \nonum \\
&= -\upsilon - \frac{\lal(1-\kappa)}{2(1-\alpha)^2\kappa} \alpha - \frac{\lal(1-\kappa)}{2(1-\kappa\alpha)}+O(b^{-1/2})\label{eq:mu01}.
\end{align}
We also obtain, using the formula for the variance of a Poisson random variable
\begin{align}
\text{Var}{\xi_0^1}  &= \log^2(\rho)\kappa b \alpha + f^2(-\upsilon)(1-\kappa)b + f^2(-\upsilon)\kappa b(1-\alpha)\nonum\\
&\myeq{(a)} \frac{\lal \alpha (1-\kappa)}{(1-\alpha)^2\kappa} + \frac{(1-\kappa)\lal}{1-\kappa\alpha} + O(b^{-1/2}),\label{eq:var01}
\end{align}
where in (a) we used (\ref{eq:logrho2}) and (\ref{eq:fups2}). Comparing (\ref{eq:mu01}) and (\ref{eq:var01}), after letting $b \to \infty$ with $\mu^{(1)}$ in (\ref{eqrecu}) using $\mu^{(0)} = 0,$ we can verify the mean and variance recursions. Next we use Lemma \ref{l:berryessen} to prove gaussianity. Note that we can express $\xi_0^1 - h$ as the Poisson sum of iid mixture random variables as follows
\[
\xi_0^1 - h = \sum_{i = 1}^{L_0} X_i,
\]
where $L_0 \sim \text{Poi}(b),$ and $\mathcal L(X_i) = \kappa \alpha \mathcal L(\log(\rho)) + (1-\kappa) \mathcal L(f(-\upsilon)) + (\kappa (1-\alpha))\mathcal L(f(-\upsilon)),$ keeping in mind the independent splitting property of Poissons, where $\mathcal L$ denotes the law of a \gls{rv}\footnote{Clearly $X_i$ are iid with mean $\mu = \kappa \alpha \log(\rho)+ (1-\kappa\alpha)f(-\nu) = \Omega(1/\sqrt{b})$ and $\sigma^2 = \Omega(1/b),$ both of which are bounded (fixed $b$ and as $n \to \infty$). Also $\mu^2 + \sigma^2 = \Omega(1/b).$}. Next we calculate $\mathbb E (|X_i|^3).$ It is easy to show using (\ref{eq:fups}) and (\ref{eq:logrho}) that
\begin{dmath}
\mathbb E (|X_i|^3) = \kappa \alpha \log^3(b) + (1-\kappa\alpha)|f^3(-\upsilon)|
= O(b^{-3/2}).
\end{dmath}
Therefore the upper bound of Lemma \ref{l:berryessen} with $\lambda = b$ becomes
\begin{align*}
\frac{C_{BE}\mathbb E(|X_i|^3)}{\sqrt{\gamma(\mu^2 + \sigma^2)^3}} &= \frac{O(b^{-3/2})}{\sqrt{b\Omega(b^{-3})}}
= O(b^{-1/2}).
\end{align*}
By Lemma \ref{l:berryessen}, taking $b \to \infty$ we obtain the convergence to Gaussian.

Having shown the induction hypothesis for $t=0,$ we now assume it holds for some $t>0.$  By using (\ref{eq:fx}), (\ref{eq:expfx}) and Lebesgue's dominated convergence theorem \cite[Theorem~16.4]{Bilsley} we obtain
\begin{dmath}
\label{eq:Efx1}
\mathbb{E} f(\xi_1^t) = (\rho-1) \mathbb{E} \left(\frac{e^{\xi_1^t}}{1+e^{\xi_1^t}}\right) - \frac{(\rho-1)^2}{2} \mathbb{E}\left(\frac{e^{2\xi_1^t}}{(1+e^{\xi_1^t})^2}\right) + O(b^{-3/2}),
\end{dmath}
and by using Lemma \ref{l:changmea} in addition we obtain
\begin{dmath}
\label{eq:Efxo}
\mathbb{E} f(\xi_0^t) = (\rho-1)\frac{\kappa(1-\alpha)}{1-\kappa} \mathbb{E} \left(\frac{1}{1+e^{\xi_1^t}}\right) -  \frac{(\rho-1)^2\kappa (1-\alpha)}{2(1-\kappa)} \mathbb{E}\left(\frac{e^{\xi_1^t}}{(1+e^{\xi_1^t})^2}\right)+ O(b^{-3/2}).
\end{dmath}
Now we take the expectation of both sides of (\ref{eq:d1}) and (\ref{eq:d2}). Using the fact that $\mathbb{E} \sum_{i=1}^L X_i = \mathbb{E}X_i \mathbb{E}L$ if $L\sim \text{Poi}$ and $X_i$ are independent and identically distributed (iid) \gls{rv}, we obtain
\begin{dmath}
\label{eq:Exo}
\mathbb{E}(\xi_0^{t+1}) = h + \log\left(\frac{p}{q}\right) \kappa b \alpha + \mathbb{E} \left(f(\xi_0^t)\right) (1-\kappa)b + \mathbb{E}\left(f(\xi_1^{t})\right) \kappa b(1-\alpha)
\end{dmath}
and
\begin{dmath}
\label{eq:Ex1}
\mathbb{E}(\xi_1^{t+1}) = h + \log\left(\frac{p}{q}\right) \kappa a \alpha + \mathbb{E} \left(f(\xi_0^t)\right) (1-\kappa)b + \mathbb{E}\left(f(\xi_1^{t})\right) \kappa a(1-\alpha).
\end{dmath}
We now substitute (\ref{eq:Efxo}) and (\ref{eq:Efx1}) in (\ref{eq:Exo}) to get:
\begin{eqnarray*}
\lefteqn{\mathbb{E}(\xi_0^{t+1})}\\
&=& h + \kappa b \alpha \log(\rho)\\
&+& (1-\kappa)b\bigg[ (\rho-1)\frac{\kappa(1-\alpha)}{1-\kappa} \mathbb{E} \left(\frac{1}{1+e^{\xi_1^t}}\right)\\  &-&\frac{(\rho-1)^2\kappa (1-\alpha)}{2(1-\kappa)} \mathbb{E}\left(\frac{e^{\xi_1^t}}{(1+e^{\xi_1^t})^2}\right)+ O(b^{-3/2}) \bigg] \\
&+&  \kappa b (1-\alpha)\bigg[ (\rho-1) \mathbb{E} \left(\frac{e^{\xi_1^t}}{1+e^{\xi_1^t}}\right) \\
&-& \frac{(\rho-1)^2}{2} \mathbb{E}\left(\frac{e^{2\xi_1^t}}{(1+e^{\xi_1^t})^2}\right) + O(b^{-3/2}) \bigg],
\end{eqnarray*}
which on simplifying and grouping like terms gives
\begin{dmath*}
\mathbb{E}(\xi_0^{t+1}) = h + \kappa b \alpha \log(\rho)+ \kappa (a-b) (1-\alpha) - \frac{\lal(1-\kappa)}{2(1-\alpha)\kappa} \mathbb E \left(\frac{e^{\xi_1^t}}{1 + e^{\xi_1^t}}\right) + O(b^{-1/2}).
\end{dmath*}
Substituting $h = -\kappa(a-b) - \log\left( \frac{1-\kappa}{\kappa(1-\alpha)}\right),$ we get
\begin{dmath*}
\mathbb{E}(\xi_0^{t+1}) = - \log\left( \frac{1-\kappa}{\kappa(1-\alpha)}\right) -\alpha\kappa(a-b) + \kappa b \alpha \log(\rho) - \frac{\lal(1-\kappa)}{2\kappa(1-\alpha)} \mathbb E \left(\frac{e^{\xi_1^t}}{1 + e^{\xi_1^t}}\right) + O(b^{-1/2}).
\end{dmath*}
Using (\ref{eq:logrho}) we get

\begin{eqnarray*}
\lefteqn{-\alpha \kappa(a-b)+ \kappa b \alpha\log(\rho)} \\
&=&  \kappa b \alpha (\log(\rho) - (\rho - 1))\\
&=& \kappa b \alpha  \left(-\frac{\lal(1-\kappa)}{2\kappa^2 b(1-\alpha)^2} + O(b^{-3/2})\right)\\
&=& -\frac{\lal \alpha (1-\kappa)}{2(1-\alpha)^2\kappa} + O(b^{-1/2}).
\end{eqnarray*}
Finally we obtain
\begin{dmath}
\label{eq:Ex0_1}
\mathbb{E}(\xi_0^{t+1}) = -\log\left(\frac{1-\kappa}{\kappa(1-\alpha)}\right) -\frac{\lal\alpha (1-\kappa)}{2(1-\alpha)^2\kappa} -
 \lal\frac{(1-\kappa)}{2(1-\alpha)\kappa} \mathbb{E}\left(\frac{e^{\xi_1^t}}{1+ e^{\xi_1^t}}\right) + O(b^{-1/2}).
\end{dmath}
Using exactly the same simplifications we can get

\begin{dmath}
\label{eq:Ex1_1}
\mathbb{E}(\xi_1^{t+1}) = -\log\left(\frac{1-\kappa}{\kappa(1-\alpha)}\right) +\frac{\alpha \lal (1-\kappa)}{2\kappa(1-\alpha)^2} +
 \frac{\lal(1-\kappa)}{2\kappa(1-\alpha)} \mathbb{E}\left(\frac{e^{\xi_1^t}}{1+ e^{\xi_1^t}}\right) + O(b^{-1/2}).
\end{dmath}
Our next goals are to compute $\textnormal{var}(\xi_0^{t+1})$ and $\textnormal{var}(\xi_1^{t+1}).$ Towards this,
observe that $f^2(x) = (\rho-1)^2 \left( \frac{e^x}{1 + e^x}\right)^2 + O(b^{-3/2}).$ Therefore
\[
\mathbb{E}(f^2(\xi_0^t)) = (\rho-1)^2 \mathbb{E}\left(\frac{e^{2\xi_0^t}}{(1+e^{\xi_0^t})^2}\right) + O(b^{-3/2}),
\]
and using Lemma \ref{l:changmea} the above becomes
\begin{equation}\label{eq:eqf20}
\mathbb{E}(f^2(\xi_0^t)) = (\rho-1)^2 \frac{\kappa (1-\alpha)}{1-\kappa} \mathbb{E} \left(\frac{e^{\xi_1^t}}{(1+e^{\xi_1^t})^2}\right) + O(b^{-3/2}).
\end{equation}
Similarly,
\begin{equation}\label{eq:eqf21}
\mathbb{E}\left(f^2(\xi_1^t)\right) = (\rho-1)^2  \mathbb{E} \left(\frac{e^{2\xi_1^t}}{(1+e^{\xi_1^t})^2}\right) + O(b^{-3/2}).
\end{equation}
Now we use the formula for the variance of Poisson sums $\text{Var}{ \sum_{i=1}^L X_i} = \mathbb{E}(X_i^2) \mathbb{E}(L)$ to get
\begin{dmath*}
\textnormal{Var}(\xi_0^{t+1}) = \log^2(\rho) \kappa b\alpha + (1-\kappa)b \mathbb E (f^2(\xi_0^t))+ \kappa b(1-\alpha) \mathbb E (f^2(\xi_1^t))
\end{dmath*}
\begin{dmath*}
\textnormal{Var}(\xi_1^{t+1}) = \log^2(\rho)\kappa a\alpha + (1-\kappa)b \mathbb E (f^2(\xi_0^t)) +
\kappa a (1-\alpha) \mathbb E (f^2(\xi_1^t)).
\end{dmath*}
Substituting (\ref{eq:eqf20}) and (\ref{eq:eqf21}) into the above equations and letting $b \to \infty,$ we get
\[
\lim_{b\to\infty}\textnormal{Var}(\xi_1^{t+1}) = \lim_{b\to\infty} \textnormal{Var}(\xi_0^{t+1}) = \mu^{(t+1)},
\]
where
\begin{dmath}
\label{eq:varx1}
\mu^{(t+1)}= \frac{\lal \alpha(1-\kappa)}{\kappa(1-\alpha)^2} + \frac{\lal(1-\kappa)}{\kappa(1-\alpha)} \mathbb{E} \left(\frac{\exp{\xi_1^t}}{1 + \exp(\xi_1^{t})}\right).
\end{dmath}
Using $\mu^{(t+1)}$ of (\ref{eq:varx1}) in (\ref{eq:Ex0_1}) and (\ref{eq:Ex1_1}) we get
\begin{align}
\mathbb E (\xi_0^{t+1}) &= -\log \left( \frac{(1-\kappa)}{\kappa (1-\alpha)}\right) - \frac{1}{2}\mu^{(t+1)}+ O(b^{-1/2}) \nonumber \\
\mathbb E (\xi_1^{t+1}) &= -\log \left( \frac{(1-\kappa)}{\kappa (1-\alpha)}\right) + \frac{1}{2}\mu^{(t+1)}+ O(b^{-1/2}). \label{eq:expxi1}
\end{align}
Now we use the fact the induction assumption that $\xi_1^{t} \to \mathcal{N}(\mathbb E(\xi_1^t),\mu^{(t)}).$ Since the function $ e^{\xi_1^t}/(1 + e^{\xi_1^t})$ is bounded, by Lebesgue's dominated convergence theorem \cite[Theorem~16.4]{Bilsley} this means $\mathbb E (1/(1 + e^{-\xi_1^t})) \to \mathbb E (1/(1 + e^{-\mathcal{N}(\mathbb E(\xi_1^t),\mu^{(t)})}))$ as $b\to\infty.$We can write $\mathcal{N}(\mathbb E(\xi_1^t),\mu^{(t)}) = \sqrt{\mu^{(t)}} Z + \mathbb E(\xi_1^t),$ where $Z \sim \mathcal N (0,1).$ Therefore we obtain
\begin{align*}
\mathbb E \left(\frac{1}{1 + e^{-\xi_1^t}}\right) &= \mathbb E \left(\frac{1}{1 + e^{-\sqrt{\mu^{(t)}}Z}\frac{(1-\kappa)}{\kappa (1-\alpha)} e^{-\frac{\mu^{(t)}}{2}}}\right)\\
&= \mathbb E \left(\frac{\kappa (1-\alpha)}{\kappa (1-\alpha) + (1-\kappa)e^{(-\sqrt{\mu^{t}}Z - \frac{\mu^{(t)}}{2})}}\right).
\end{align*}
Substituting the above into (\ref{eq:varx1}) gives us the recursion for $\mu^{(t+1)}$ given in (\ref{eqrecu}).

Next we prove Gaussianity. Consider

\begin{eqnarray}
\label{eq:intst}
\lefteqn{\xi_0^{t+1} - \mathbb{E} (\xi_0^{t+1})}\nonum \\ &=& \log\left(\frac{p}{q}\right) (L_{0c} - \mathbb{E}(L_{0c})) + \sum_{i=1}^{L_{00}} (f(\xi_{0,i}^t) -\mathbb{E} (f(\xi_0^t))) + \nonum \\ && \sum_{i=1}^{L_{01}} (f(\xi_{1,i}^t) -\mathbb{E} (f(\xi_1^t)) ) +  (L_{00} - \mathbb{E}(L_{00})) \mathbb{E} (f(\xi_0^t)) + \nonum \\&& (L_{01} - \mathbb{E}(L_{01})) \mathbb{E} (f(\xi_1^t)).
\end{eqnarray}

Let us look at the second term. Let $X_i = f(\xi_{0,i}^t) -\mathbb{E} f(\xi_{0,i}^t).$ Then it can be shown that $\mathbb{E}X_i^2 = O(1/b).$ Let $D := \sum_{i=1}^{L_{00}} X_i - \sum_{i=1}^{\mathbb{E}L_{00}} X_i.$ In the second term the summation is taken up to  $i \le \mathbb{E}L_{00}.$ Then $\mathbb E (D^2) = |\sum_{i=1}^\delta X_i|^2,$ where $\delta \le  | L_{00} - \mathbb{E}L_{00}|+1,$ where the extra 1 is because $\mathbb{E}L_{00}$ may not be an integer. Therefore $\mathbb{E}D^2 = \mathbb{E}\delta \mathbb{E}|X_1|^2 \le (C/b) ((1-\kappa) b+1)^{1/2} = O(1/\sqrt{b}).$ Thus, we can replace the Poisson upper limits of the summations in the second and third terms of (\ref{eq:intst}) by their means, leading to
\begin{dmath}
\xi_0^{t+1} - \mathbb{E} (\xi_0^{t+1}) = \log\left(\frac{p}{q}\right) (L_{0c} - \mathbb{E}(L_{0c})) + \sum_{i=1}^{\mathbb{E}(L_{00})} (f(\xi_{0,i}^t) -\mathbb{E} (f(\xi_0^t)) ) + \sum_{i=1}^{\mathbb{E}(L_{01})} (f(\xi_{1,i}^t) -\mathbb{E} (f(\xi_1^t))) + (L_{00} - \mathbb{E}(L_{00})) \mathbb{E} f(\xi_0^t) + (L_{01} - \mathbb{E}(L_{01})) \mathbb{E} (f(\xi_1^t)) + o_p(1),
\end{dmath}
where $o_p(1)$ indicates a \gls{rv} that goes to zero in probability in the limit. The combined variance of all other terms approaches $\mu^{(t+1)},$ defined in (\ref{eqrecu}), as $b\to\infty$ and it is finite for a fixed $t.$ Now since we have an infinite sum of independent \gls{rv}s as $a,b \to \infty$, with zero mean and finite variance, from the standard CLT, we can conclude that the distribution tends to $\mathcal{N}(0,\mu^{t+1}).$ The argument for $\xi_1^{t+1}$ is identical.\end{proof}
\section{Finishing the proof of Theorem \ref{pr:weakrecovery}}\label{appr2}
\begin{proof}{
We bound $\mathbb E (|\overline{S} \Delta \widehat{S}_0|)/(K(1-\alpha))$ as follows:
\begin{align}
\lim_{b\to \infty}\lim_{n\to\infty}\frac{\mathbb{E} (|\overline{S}\Delta \widehat S_0|)}{K(1-\alpha)}  &= \lim_{b\to \infty}\lim_{n\to\infty} \left(\frac{\mathbb E \left(\sum_{i=1}^n \mathbf{1}_{\sigma_i \neq \widehat{\sigma}_i} \right)}{K-K\alpha}\right)  \nonumber \\
&\le \lim_{b\to \infty} \biggl(\frac{(1-\kappa)}{\kappa(1-\alpha)}\mathbb P(\xi_0^t \ge 0) + \nonum\\
& \quad \quad \mathbb P(\xi_1^t \le 0)\biggr)\label{eq:frerror},
\end{align}
since
\begin{eqnarray}
\lefteqn{\mathbb E \left(\sum_{i=1}^n \mathbf{1}_{\sigma_i \neq \widehat{\sigma}_i} \right) = }\nonum \\
&&n (\mathbb P (c_i = 0,\sigma_i = 0) \mathbb P(R_i^t > \upsilon|c_i = 0,\sigma_i = 0)+ \nonum \\
&&\mathbb P (c_i = 0,\sigma_i = 1) \mathbb P(R_i^t < \upsilon|c_i = 0,\sigma_i = 1)),\label{eq:Ri}
\end{eqnarray}
and since $R_i^t - R^t_{i \to u} = O(b^{-1/2}).$ Indeed, given the $b\to\infty$ limit in (\ref{eq:frerror}), the bound $O(b^{-1/2})$ allows us to replace $R_i^t$ in (\ref{eq:Ri}) by the distribution limit when $n \to \infty,$ which is $\xi_0^t$ or $\xi_1^t$ when conditioned on $\{\sigma_i = 0\}$ or $\{\sigma_i = 1\}$ respectively, for an arbitrary $i.$ We now analyze each term in (\ref{eq:frerror}) separately. By Proposition \ref{prop:largedegreeasyms} we have
\[
\lim_{b\to \infty}  \mathbb P(\xi_1^t \le 0) = Q\left(\frac{1}{\sqrt{\mu^{(t)}}}\left(\frac{\mu^{(t)}}{2} - \log\frac{(1-\kappa)}{\kappa(1-\alpha)}\right)\right)
\]
where $Q(\cdot)$ denotes the standard $Q$ function. Notice that by (\ref{eqrecu}) we have that $\mu^{(t)} \ge \lal \alpha (1-\kappa)/(\kappa(1-\alpha)^2),$ since $\mathbb{E} \left(\frac{1-\kappa}{\kappa(1-\alpha) + (1-\kappa) \exp(-\mu/2 - \sqrt {\mu} Z)}\right) \ge 0.$  In addition, by (\ref{eq:varx1}), $\mu^{(t)} \le \frac{\lal(1-\kappa)}{\kappa(1-\alpha)^2}.$ Note that the lower bound on $\mu^{(t)}$ is not useful when $\alpha = 0.$ Therefore by using the Chernoff bound for the $Q$ function, $Q(x) \le \frac{1}{2}e^{-x^2/2},$ we get
\begin{align}
\lim_{b\to\infty}\mathbb P(\xi_1^t \le 0) &\le \frac{1}{2}e^{-\frac{1}{2\mu^{(t)}} (\frac{\mu^{(t)}}{2} - \log\left(\frac{1-\kappa}{\kappa(1-\alpha)}\right))^2}\nonum \\
&= \frac{1}{2}e^{-\frac{\mu^{(t)}}{8} (1 - \frac{2}{\mu^{(t)}}\log\left(\frac{1-\kappa}{\kappa(1-\alpha)}\right))^2}\nonum \\
&\le \frac{1}{2}e^{-\frac{\mu^{(t)}}{8}} e^{\frac{1}{2} \log(\frac{1-\kappa}{\kappa(1-\alpha)})}\nonum \\
&= \frac{1}{2}\sqrt{\frac{1-\kappa}{\kappa(1-\alpha)}}e^{-\frac{\mu^{(t)}}{8}},\label{eq:p1}
\end{align}
where we used the fact that $(1-x)^2 \ge 1-2x$ for any $x>0.$
By employing similar reductions, we can show
\begin{align}\label{eq:p2}
\lim_{b\to\infty} \left(\frac{(1-\kappa)}{\kappa(1-\alpha)}\right)\mathbb P(\xi_0^t \ge 0) &\le \frac{1}{2}\sqrt{\frac{1-\kappa}{\kappa(1-\alpha)}}e^{-\frac{\mu^{(t)}}{8}}.
\end{align}

Substituting (\ref{eq:p1}) and (\ref{eq:p2}) back in (\ref{eq:frerror}) and using the fact that $\mu^{(t)} \ge \lal \alpha (1-\kappa)/(\kappa(1-\alpha)^2),$ we get
\[
\lim_{b\to\infty}\lim_{n\to\infty}\frac{\mathbb{E} (|\overline{S}\Delta \widehat S_0|)}{K(1-\alpha)} \le \sqrt{\frac{1-\kappa}{\kappa(1-\alpha)}}e^{-\frac{\lal\alpha(1-\kappa)}{8\kappa(1-\alpha)^2}}
\]
Then using (\ref{eq:boundexp}) we get the desired result in (\ref{eq:prweak}) .}\end{proof}


\bibliographystyle{IEEEtran}
\bibliography{myrefs.bib}
\newpage
\section*{Supplementary Material}
\begin{proof}[Proof of Proposition \ref{l:bpdistr2}]
We derive the conditional distributions of the messages $R_{u\to i}^t$ for a finite $t$ given $\{\sigma_u = 0\}$ and given $\{\sigma_u = 1\}.$ In this limit the tree coupling of Lemma \ref{l:coupling} holds with a slightly modified construction of the tree to accomodate the difference in the generation of cued nodes.
It is similar to the tree coupling in Lemma \ref{l:coupling}, with the only difference being the generation of cues. At any level of the tree, a node $u$ is labelled a cue such that $\mathbb P (c_u = 1|\tau_u = 1) = \alpha \beta$ and $\mathbb P (c_u = 1|\tau_u = 0) = \kappa\alpha (1-\beta)/(1-\kappa),$ so that the equalities in (\ref{eq:alpha}) and (\ref{def:beta}) hold, where $c_u$ denotes the cue membership of node $u$ on the tree.
 Let $F_{u \to i}^t$ be such that $R_{u \to i}^t = F_{u \to i}^t + h_u,$ for any two neighbouring nodes $i$ and $u.$ Then, it can be seen from (\ref{eq:bpupdate2}) that $F_{u \to i}^t$ satisfies the following recursion
\begin{equation}\label{fj}
F_{u \to i}^{t+1} = -\kappa(a-b) + \sum_{l\in\delta u, l\neq i} f_{\textnormal{isi}}(F_l^{t}+h_l),
\end{equation}
where $f_{\textnormal{isi}}(x) := \log\left(\frac{e^{(x -\nu)}\rho + 1}{e^{(x-\nu)}+1}\right).$ Let $\Psi^t_0,\Psi^t_1$ be the \gls{rv}s that have the conditional asymptotic distribution of $F_{u \to i}^t$ given $\{\sigma_u = 0\}$ and $\{\sigma_u = 1\}$ respectively in the limit $n \to \infty.$ Then, by studying the recursion (\ref{fj}) on the tree we can conclude that $\Psi^t_0,\Psi^t_1$ satisfy the following recursive distributional equations
\begin{align}
\Psi_0^{t+1} &\myeq{D} -\kappa(a-b) + \sum_{i = 0}^{L_{01c}} f_{\textnormal{isi}}(\Psi_{1i}^t + B_c) +\nonum \\ &\sum_{i = 0}^{L_{01n}}f_{\textnormal{isi}}(\Psi_{1i}^t + B_n) + \sum_{i=0}^{L_{00c}} f_{\textnormal{isi}}(\Psi_{0i}^t + B_c) + \nonum \\ &\sum_{i=0}^{L_{00n}} f_{\textnormal{isi}}(\Psi_{0i}^t + B_n),\label{eq:recpsi1}\\
\Psi_1^{t+1} &\myeq{D} -\kappa(a-b) + \sum_{i = 0}^{L_{11c}} f_{\textnormal{isi}}(\Psi_{1i}^t + B_c) +\nonum \\&\sum_{i = 0}^{L_{11n}}f_{\textnormal{isi}}(\Psi_{1i}^t + B_n) + \sum_{i=0}^{L_{10c}} f_{\textnormal{isi}}(\Psi_{0i}^t + B_c) + \nonum\\&\sum_{i=0}^{L_{10n}} f_{\textnormal{isi}}(\Psi_{0i}^t + B_n),\label{eq:recpsi2}
\end{align}
where $\myeq{D}$ represents equality in distribution, and the random sums are such that $L_{01c} \sim \text{Poi}(\kappa b \alpha \beta), L_{01n} \sim \text{Poi}(\kappa b(1-\alpha\beta)), L_{00c} \sim \text{Poi}(\kappa b \alpha(1-\beta)), L_{00n} \sim \text{Poi}(b(1-\kappa - \kappa\alpha(1-\beta)), L_{11c} \sim \text{Poi}(\kappa a \alpha \beta) , L_{11n} \sim \text{Poi}(\kappa a(1-\alpha\beta)), L_{10c} \sim \text{Poi}(\kappa b \alpha(1-\beta)) ,$ and $L_{10n} \sim \text{Poi}(b(1-\kappa - \kappa\alpha(1-\beta))),$ $B_c = \log\left(\frac{\beta(1-\kappa)}{(1-\beta)\kappa}\right);$ $B_n = \log\left(\frac{(1-\alpha\beta)(1-\kappa)}{(1-\kappa-\alpha\kappa+\alpha\kappa\beta)}\right);$ and $\Psi_{0,i}^t$ and $\Psi_{1,i}^t$ are iid \gls{rv}s with the same distribution as $\Psi_0^t$ and $\Psi_1^t$ respectively.

We now derive the asymptotic distributions $\Psi_0^{t+1}$ and $\Psi_1^{t+1}$ when $a,b \to \infty$ such that $\lambda = \frac{\kappa^2(a-b)^2}{(1-\kappa)b}$ and $\kappa$ are fixed. Observe that $\rho = 1 + \sqrt{\frac{\lambda(1-\kappa)}{\kappa^2 b}} = 1 + \sqrt{\frac{r}{b}},$ where $r := \frac{\lambda(1-\kappa)}{\kappa^2}.$
Notice that if $P_0 \sim \mathcal{L}(\Psi_0^{t})$ and $P_1 = \mathcal{L}(\Psi_1^{t}),$ we have, since $\Psi = \log\left(\frac{dP_1}{dP_0}\right),$ that $\frac{dP_0}{dP_1}(\Psi) = e^{-\Psi}.$ Also
\begin{eqnarray}
\lefteqn{f_{\textnormal{isi}}(x) = \log\left(1 + (\rho-1)\frac{e^{x-\nu}}{1 + e^{x-\nu}}\right)}\\
&=& \sqrt{\frac{r}{b}}\left(\frac{e^{x-\nu}}{1 + e^{x-\nu}}\right) - \\&&\frac{1}{2}\frac{r}{b}\left(\frac{e^{x-\nu}}{1 + e^{x-\nu}}\right)^2 + O(b^{-3/2}),
\end{eqnarray}
and
\begin{dmath}
f^2_{\textn{isi}}(x) = \frac{r}{b}\left(\frac{e^{x-\nu}}{1+e^{x-\nu}}\right)^2 + O(b^{-3/2}).
\end{dmath}
Now we can reformulate the recursions in (\ref{eq:recpsi1}) and (\ref{eq:recpsi2}) as a Poisson sum as follows:
\begin{align}
\Psi_0^{t+1} &\myeq{D} -\kappa(a-b) + \sum_{l=1}^{L_{0}} X_l \label{eq:p1}\\
\Psi_1^{t+1} &\myeq{D} -\kappa(a-b) + \sum_{l=1}^{L_{1}} Y_l\label{eq:p2},
\end{align}
where $L_{0} = \text{Poi}(b),L_{1} = \text{Poi}(\kappa a+(1-\kappa)b)$ and $X_l$ and $Y_l$ are mixture \gls{rv}s with laws defined as follows:
\begin{dmath*}
\mathcal{L}(X_l) = \alpha\kappa(1-\beta) \mathcal{L}(f_{\textnormal{isi}}(\Psi_0^t+B_c)) + (1-\kappa)(1-\alpha(1-\beta)e^{-\nu})\mathcal{L}(f_{\textnormal{isi}}(\Psi_0^t+B_n))+\alpha\kappa\beta \mathcal{L}(f_{\textnormal{isi}}(\Psi_1^t+B_c)) + \kappa(1-\alpha\beta)\mathcal{L}(f_{\textnormal{isi}}(\Psi_1^t+B_n)),
\end{dmath*}
\begin{dmath*}
  \mathcal{L}(Y_l) = \frac{\alpha\kappa b(1-\beta)}{\kappa a + (1-\kappa)b}\mathcal{L}(f_{\textnormal{isi}}(\Psi_0^t+B_c)) + \frac{(1-\kappa)b(1-\alpha(1-\beta)e^{-\nu})}{\kappa a + (1-\kappa)b}\mathcal{L}(f_{\textnormal{isi}}(\Psi_0^t+B_n)) + \frac{\kappa a \alpha\beta}{\kappa a + (1-\kappa)b}\mathcal{L}(f_{\textnormal{isi}}(\Psi_1^t+B_c)) + \frac{\kappa a (1-\alpha\beta)}{\kappa a + (1-\kappa)b} \mathcal{L}(f_{\textnormal{isi}}(\Psi_1^t+B_n)).
\end{dmath*}
Observe that we have $B_c - \nu = \log(\frac{\beta}{1-\beta})$ and $B_n - \nu = \log(\frac{\kappa(1-\alpha\beta)}{(1-\kappa -\alpha\kappa(1-\beta))}).$
We can calculate $\mathbb{E}(X_l)$ as
\begin{dmath*}
\mathbb E(X_l) = \alpha\kappa\beta\sqrt{\frac{r}{b}}  + \kappa(1-\alpha\beta)\sqrt{\frac{r}{b}}- \alpha\kappa\beta\frac{r}{2b} \mathbb{E}\left(\frac{e^{\Psi_1^t+B_c-\nu}}{1+e^{\Psi_1^t+B_c-\nu}}\right) - \kappa(1-\alpha\beta)\frac{r}{2b}\mathbb{E}\left(\frac{e^{\Psi_1^t+B_n-\nu}}{1+e^{\Psi_1^t+B_n-\nu}}\right) + O(b^{-3/2}),
\end{dmath*}
which gives,
\begin{dmath*}
\mathbb E(X_l) =   \kappa\sqrt{\frac{r}{b}}- \alpha\kappa\beta\frac{r}{2b} \mathbb{E}\left(\frac{e^{\Psi_1^t+B_c-\nu}}{1+e^{\Psi_1^t+B_c-\nu}}\right) - \kappa(1-\alpha\beta)\frac{r}{2b}\mathbb{E}\left(\frac{e^{\Psi_1^t+B_n-\nu}}{1+e^{\Psi_1^t+B_n-\nu}}\right) + O(b^{-3/2}).
\end{dmath*}
Similarly
\begin{dmath*}
\mathbb E(X_l^2) = \alpha\kappa \beta \frac{r}{b}\mathbb{E}\left(\frac{e^{\Psi_1^t+B_c-\nu}}{1+ e^{\Psi_1^t+B_c-\nu}}\right) + \frac{r\kappa(1-\alpha\beta)}{b}\mathbb{E}\left(\frac{e^{\Psi_1^t+B_n -\nu}}{1+e^{\Psi_1^t+B_n-\nu}}\right) + O(b^{-3/2}),
\end{dmath*}
and
\begin{eqnarray}
  \lefteqn{\mathbb E(|X_l^3|)}\nonum \\
  &=& \alpha \kappa \beta (\frac{r}{b})^{3/2}\mathbb E\left(\frac{e^{2(\Psi_1^t+B_c-\nu)}}{(1+e^{\Psi_1^t+B_c-\nu})^2}\right) + \nonum \\ &&\frac{\kappa(1-\alpha\beta) r^{3/2}}{b^{3/2}}\mathbb E \left(\frac{e^{2(\Psi_1^t+B_n-\nu)}}{(1+e^{\Psi_1^t+B_n-\nu})^2}\right)\nonum  \\
  && + O(b^{-2}). \label{eq:expXi3}
\end{eqnarray}
Similarly we can calculate the moments of $Y_l$ as follows:
\begin{dmath*}
  \mathbb E (Y_l) = \frac{\alpha\kappa b \beta}{\kappa a + (1-\kappa)b}\sqrt{\frac{r}{b}}\mathbb{E}\left(\frac{1+\rho e^{\Psi_1^t + B_c - \nu}}{1 + e^{\Psi_1^t + B_c - \nu}}\right) - \frac{r\alpha\kappa\beta}{2(\kappa a + (1-\kappa)b)}\mathbb{E}\left(\frac{e^{\Psi_1^t+B_c-\nu}(1+\rho e^{\Psi_1^t+B_c-\nu})}{(1+e^{\Psi_1^t+B_c-\nu})^2}\right) + \frac{\kappa b(1-\alpha\beta)}{\kappa a + (1-\kappa)b}\sqrt{\frac{r}{b}}\mathbb{E}\left(\frac{1+\rho e^{\Psi_1^t+B_n-\nu}}{1+e^{\Psi_1^t+B_n-\nu}}\right) - \frac{r\kappa(1-\alpha\beta)}{2(\kappa a + (1-\kappa b))} \mathbb{E}\left(\frac{e^{\Psi_1^t+B_n-\nu}(1+\rho e^{\Psi_1^t+B_n-\nu})}{(1+e^{\Psi_1^t+B_n-\nu})^2}\right) + O(b^{-3/2}),
\end{dmath*}
giving
\begin{dmath*}
\mathbb E (Y_l) = \kappa \sqrt{rb}\frac{1}{\kappa a + (1-\kappa)b} +  \frac{r\alpha\beta\kappa}{2(\kappa a + (1-\kappa)b)}\mathbb E \left(\frac{e^{\Psi_1^t+B_c-\nu}}{1+e^{\Psi_1^t+B_c-\nu}}\right) + \frac{r\kappa(1-\alpha\beta)}{2(\kappa a + (1-\kappa)b)}\mathbb E \left(\frac{e^{\Psi_1^t+B_n-\nu}}{1+e^{\Psi_1^t+B_n-\nu}}\right) + O(b^{-3/2}).
\end{dmath*}
In addition,
\begin{dmath*}
\mathbb E(Y_l^2) = \frac{\alpha\kappa\beta r}{\kappa a + (1-\kappa)b} \mathbb E \left(\frac{e^{\Psi_1^t+B_c-\nu}}{1+e^{\Psi_1^t+B_c-\nu}}\right) + \frac{\kappa r(1-\alpha\beta)}{\kappa a + (1-\kappa)b} \mathbb E \left(\frac{e^{\Psi_1^t+B_n-\nu}}{1+e^{\Psi_1^t+B_n-\nu}}\right) \\
+ O(b^{-2})
\end{dmath*}
and
\begin{eqnarray}
\lefteqn{\mathbb E(|Y_l^3|)} \nonum \\
&=& \frac{\alpha\kappa \beta r^{3/2}}{(\kappa a + (1-\kappa)b)b^{1/2}} \mathbb{E}\left(\frac{e^{2\Psi_1^t+2B_c-2\nu}}{(1+e^{\Psi_1^t+B_c-\nu})^2}\right)\nonum \\ &&+\frac{\kappa(1-\alpha\beta)r^{3/2}}{(\kappa a + (1-\kappa)b) b^{1/2}}\mathbb E \left(\frac{e^{2\Psi_1^t+2B_n - 2\nu}}{(1+e^{\Psi_1^t+2B_n-\nu})^2}\right)\nonum \\&&+ O(b^{-2}).\label{eq:expYi3}
\end{eqnarray}
Let us define $\mu^{(t)}$ as
\begin{dmath}
  \mu^{(t+1)} = \alpha \beta \kappa r \mathbb E \left(\frac{1}{1 + e^{-\Psi_1^t-B_c + \nu}}\right) + \kappa r (1-\alpha\beta)\mathbb E \left(\frac{1}{1 + e^{-\Psi_1^t-B_n +\nu}}\right)\label{eq:defmut}.
\end{dmath}

Finally we have
\begin{align*}
  \mathbb E(\Psi_0^{t+1}) &= -\kappa(a-b) + b \mathbb E(X_l) \\
  &= -\frac{\alpha\kappa\beta r}{2} \mathbb E\left(\frac{e^{\Psi_1^t+B_c-\nu}}{1 + e^{\Psi_1^t+B_c - \nu}}\right) \\
  & \quad -\frac{\kappa(1-\alpha\beta)r}{2} \mathbb E\left(\frac{e^{\Psi_1^t+B_n-\nu}}{1 + e^{\Psi_1^t+B_n - \nu}}\right) + O(b^{-1/2})\\
  &= -\frac{\mu^{t+1}}{2}  + O(b^{-1/2}),
\end{align*}
and
\begin{align*}
  \mathbb E(\Psi_1^{t+1}) &= -\kappa(a-b) + (\kappa a + (1-\kappa)b) \mathbb E(Y_l) \\
  &= \frac{\alpha\kappa\beta r}{2} \mathbb E\left(\frac{e^{\Psi_1^t+B_c-\nu}}{1 + e^{\Psi_1^t+B_c - \nu}}\right) \\
  & \quad + \frac{\kappa(1-\alpha\beta)r}{2} \mathbb E\left(\frac{e^{\Psi_1^t+B_n-\nu}}{1 + e^{\Psi_1^t+B_n - \nu}}\right) + O(b^{-1/2})\\
  &= \frac{\mu^{(t+1)}}{2}  + O(b^{-1/2}).
\end{align*}
In addition, for the variances of $\Psi_0^{t+1}$ and $\Psi_1^{t+1}$ we have

\begin{align}
  \textnormal{Var}(\Psi_0^{t+1}) &= b \mathbb E(X_l^2)\nonum \\
  &= \alpha\beta \kappa r \mathbb E\left(\frac{e^{\Psi_1^t+B_c - \nu}}{1 + e^{\Psi_1^t+B_c - \nu}}\right) \nonum \\ &\quad + \kappa(1-\alpha\beta)r \mathbb E \left(\frac{e^{\Psi_1^t+B_n -\nu}}{1 + e^{\Psi_i^t + B_n -\nu}}\right)
    + O(b^{-1/2})\nonum \\
    &= \mu^{(t+1)} + O(b^{-1/2}),\label{eq:vart0}
\end{align}
and similarly
\begin{align}
  \textnormal{Var}(\Psi_1^{t+1}) &= (\kappa a + (1-\kappa)b) \mathbb E(Y_l^2)\\
  &= \mu^{(t+1)} + O(b^{-1/2}).\label{eq:vart1}
\end{align}
Now we need to show the Gaussianity of the messages $\Psi_0^{t}$ and $\Psi_1^t,$ which we show using Lemma \ref{l:berryessen}.
For (\ref{eq:p1}) the upperbound in Lemma \ref{l:berryessen}  becomes
\begin{align}
\frac{C_{BE}\mathbb E(|X_i|^3)}{\sqrt{\gamma(\mu^2+\sigma^2)^3}} &= \frac{C_{BE}b \mathbb E(|X_i|^3)}{\sqrt{(b(\mu^2+\sigma^2))^3}}\nonum \\
&= \frac{C_{BE} b\mathbb E(|X_i|^3)}{\textnormal{Var}(\Psi_0^{t+1})^{3/2}} \label{eq:b1}
\end{align}
Similarly for (\ref{eq:p2}) we get
\begin{align}
\frac{C_{BE}\mathbb E(|Y_i|^3)}{\sqrt{\gamma(\mu^2+\sigma^2)^3}} &= \frac{C_{BE} (\kappa a + (1-\kappa)b)\mathbb E(|Y_i|^3)}{\textnormal{Var}(\Psi_1^{t+1})^{3/2}}.\label{eq:b2}
\end{align}

In Lemma~\ref{lem:fmu} stated and proved below, we show that $\mu^{(t+1)} \ge \alpha\beta^2\lambda\frac{1-\kappa}{\kappa}.$
Therefore for any $\kappa < 1/2,$ we have
\[
\textnormal{Var}(\Psi_0^{t+1}) = \textnormal{Var}(\Psi_0^{t+1}) \ge \frac{\alpha\beta^2\lambda}{2} + O(b^{-1/2}) = \Theta(1),
\]
under the assumptions of the proposition.
In addition we have $b\mathbb E (|X_i|^3) = O(b^{-1/2})$ and $(\kappa a + (1-\kappa)b)\mathbb E(|Y_i|^3) = O(b^{-1/2})$ from (\ref{eq:expXi3}) and (\ref{eq:expYi3}). Thus the bounds given in (\ref{eq:b1}) and (\ref{eq:b2}) both tend to zero as $b\to \infty.$

Hence by Lemma \ref{l:berryessen}, we obtain that $\Psi_1^t \to \mathcal{N}(\frac{\mu^{(t)}}{2},\mu^{(t)})$ and  $\Psi_0^t \to \mathcal{N}(-\frac{\mu^{(t)}}{2},\mu^{(t)})$ as $b \to \infty,$ where from (\ref{eq:defmut}), $\mu^{(t)}$ satisfies the following recursion with inital condition $\mu^{(0)} = 0:$
\begin{dmath}
  \mu^{(t+1)} = \alpha\beta \lambda \mathbb E \left(\frac{(1-\kappa)}{\kappa + (1-\kappa)e^{-\sqrt{\mu^{(t)}}Z -\frac{\mu^{(t)}}{2} -B_c}}\right) + (1-\alpha\beta)\lambda \mathbb E \left(\frac{(1-\kappa)}{\kappa + (1-\kappa)e^{-\sqrt{\mu^{(t)}}Z -\frac{\mu^{(t)}}{2} -B_n}}\right).\label{eq:defmut2}
\end{dmath}
Consequently, the distributions of the messages $R^t_{u \to i}$ in the limit of $n \to \infty$ converge to $\Gm_j^t + h_u,$ given $\{\sigma_u = j\},$ where $\Gm_1^t \sim \mathcal{N}(\frac{\mu^{(t)}}{2},\mu^{(t)})$ and $\Gm_0^t \sim \mathcal{N}(-\frac{\mu^{(t)}}{2},\mu^{(t)}),$ in the large degree limit where $b \to \infty.$\end{proof}
\subsection*{Proving the bound on $\mu^{(t)}$}
Let $F(\mu)$ be defined as
\begin{eqnarray*}
  \lefteqn{F(\mu)} \\&=& \alpha\beta^2 \lambda \mathbb E \left(\frac{(1-\kappa)/\kappa }{\beta + (1-\beta)\exp(-\mu/2-\sqrt{\mu}Z)}\right) + \\&&(1-\alpha\beta)^2\lambda\\&& \mathbb E \left(\frac{(1-\kappa)}{\kappa(1-\alpha\beta)+(1-\kappa -\alpha\kappa +\alpha\kappa\beta)e^{(-\mu/2-\sqrt{\mu}Z})}\right).
\end{eqnarray*}
Then $\mu^{(t)}$ satisfies the recursion $\mu^{(t+1)} = F(\mu^{(t)}),$ by substituting for $B_c$ and $B_n$ in (\ref{eq:defmut2}). Below we show a lower bound on  $F(\mu).$ For its proof we need the following Lemma from \cite{alon2004probabilistic}.
\begin{lemma}\label{lem:fkg}\cite[Theorem 6.2.1]{alon2004probabilistic}
If $f,g: \mathbb R \to \mathbb R$ are two non-decreasing functions, then $\mathbb E(fg) \ge \mathbb E(f) \mathbb E(g).$
\end{lemma}
Now we state our result on $F(\mu).$
\begin{lemma}\label{lem:fmu}
  For $0 < \beta < 1,$
  \[
  F(\mu) \ge \alpha\beta^2\lambda \frac{1-\kappa}{\kappa}.
  \]
\end{lemma}
\begin{proof}
We show that
\[
g_{\beta} =  \mathbb E\left(\frac{1}{\beta + (1-\beta)\exp(-\mu/2 - \sqrt{\mu}Z)}\right)
\]
is nonincreasing for $0 \le \beta \le 1$ as shown below.
Let $X = \exp(-\sqrt{\mu}Z).$
Then $\frac{d}{d\beta}(g_{\beta}) = \mathbb E \left(\frac{\exp(-\mu/2)X - 1}{(\beta + (1-\beta)e^{-\mu/2}X)^2}\right).$
Now we show $\frac{d}{d\beta}(g_{\beta}) < 0$ using Lemma \ref{lem:fkg}. In Lemma \ref{lem:fkg}, let $f = \exp(-\mu/2)X$ and $g = \frac{-1}{(\beta + (1-\beta)e^{-\mu/2}X)^2}.$ Clearly these are non-decreasing in $X.$ Therefore $\mathbb E(fg) \ge \mathbb E(f)\mathbb E (g) = \mathbb E (g),$ since $\mathbb E (f) = \mathbb E e^{-\mu/2}e^{\sqrt{\mu}Z} = 1.$ Therefore we have
\[
\mathbb E \left(\frac{-e^{-\mu/2}X}{(\beta + (1-\beta)e^{-\mu/2}X)^2}\right) \ge \mathbb E \left(\frac{-1}{(\beta + (1-\beta)e^{-\mu/2}X)^2}\right),
\]
hence $\frac{dg_{\beta}}{d\beta} < 0$ for all $\beta.$ Therefore $1 = g_{\beta}(1) \le g_{\beta}(\beta)$ for $\beta < 1.$ The result then follows by substituting this lower bound in the definition of $F(\mu)$ and observing that the second term is strictly non-negative.
\end{proof}

\subsection*{Proof of Theorem \ref{thm:bpimpsideinfo}}
\begin{proof}

Notice that when we set $\beta = 1$ the recursion (\ref{eq:recmu2}) becomes the same as (\ref{eqrecu}). Also, when $\beta = 0$ we can retrieve the recursion for standard BP without side-information, i.e., and from this it can be gleaned that the asymptotic error rate is zero only if $\lambda > 1/e.$

Let us now consider $0 < \beta < 1.$ By Lemma \ref{lem:fmu}, we have
\[
\alpha\beta^2\lambda \frac{1-\kappa}{\kappa} \le \mu^{(t)} \le \lambda\frac{(1-\kappa)}{\kappa}.
\]
Hence $\mu^{(t)} = \Theta\left(\frac{1-\kappa}{\kappa}\right).$
The asymptotic distributions of the messages are as follows:
\begin{align*}
\Gm_{0,0}^{t} &\sim \mathcal{N}(-\mu^{(t)}/2,{\mu^{(t)}}) + \log\left(\frac{(1-\alpha\beta)(1-\kappa)}{(1-\kappa-\alpha\kappa + \alpha\kappa\beta)}\right)\\
\Gm_{0,1}^{t} &\sim \mathcal{N}(-\mu^{(t)}/2,{\mu^{(t)}}) + \log\left(\frac{\beta(1-\kappa)}{\kappa(1-\beta)}\right)\\
\Gm_{1,0}^{t} &\sim \mathcal{N}(\mu^{(t)}/2,{\mu^{(t)}}) + \log\left(\frac{(1-\alpha\beta)(1-\kappa)}{(1-\kappa-\alpha\kappa + \alpha\kappa\beta)}\right)\\
\Gm_{1,1}^{t} &\sim \mathcal{N}(\mu^{(t)}/2,{\mu^{(t)}}) + \log\left(\frac{\beta(1-\kappa)}{\kappa(1-\beta)}\right),\\
\end{align*}
where $\Gm_{j,k}^t$ is the rv with the asymptotic distribution of the messages $R_{u \to i}^t$ in the limit of $n \to \infty$ and $b \to \infty,$ given $\{\sigma_u = j, c_u = k\}.$ We can now write the probability of error $p_e^{\beta}$ of the per-node MAP detector $\widehat{S}_0$ as
\begin{align*}
p_e^{\beta} &= p_e^{\beta}(i|\sigma_i=0,c_i=0) P(\sigma_i = 0,c_i = 0)+\\&p_e^{\beta}(i|\sigma_i=0,c_i=1) P(\sigma_i = 0,c_i = 1)\\&+p_e^{\beta}(i|\sigma_i=1,c_i=0) P(\sigma_i = 1, c_i = 0)\\&+p_e^{\beta}(i|\sigma_i=1,c_i=1) P(\sigma_i = 1, c_i = 1)\\
& = P_{0,0}(R_i^t > \nu)\pi_{0,0}+P_{0,1}(R_i^t > \nu)\pi_{0,1}\\&+P_{1,0}(R_i^t < \nu)\pi_{1,0}+P_{1,1}(R_i^t < \nu)\pi_{1,1},
\end{align*}
 is the error rate of Algorithm \ref{alg:bp_alg2}, where $p_e^{\beta}(i|\sigma_i=0,c_i=0)$ denotes the probability that node $i$ is misclassified, given $\{\sigma_i=0,c_i=0\}$ and $\pi_{0,1} = \mathbb P(\sigma_i = 0,c_i = 1)$ etc. Then the expected fraction of mislabelled nodes $\frac{\mathbb E (|\hat{S}_0\Delta S|)}{K}$ in the limit $n \to \infty, b \to \infty$  is
\begin{eqnarray*}
  \lefteqn{\lim_{b\to \infty}\lim_{n\to\infty}\frac{np_e^{\beta}}{K}} \\
   &=&Q\left(\frac{\frac{\mu^{(t)}}{2} + \log(\frac{\beta}{(1-\beta)})}{\sqrt{\mu^{(t)}}}\right)\alpha\beta +\\ &&(1-\alpha\beta)\\&&Q\left(\frac{\frac{\mu^{(t)}}{2} - \log\left(\frac{1-\kappa}{\kappa}\left(\frac{1 - \frac{\alpha\kappa(1-\beta)}{1-\kappa}}{1-\alpha\beta}\right)\right)}{\sqrt{\mu^{(t)}}}\right) +\\ &&\alpha(1-\beta) Q\left(\frac{\frac{\mu^{(t)}}{2} - \log(\frac{\beta}{1-\beta})}{\sqrt{\mu^{(t)}}}\right) +\\&& (\frac{1-\kappa}{\kappa}-\alpha(1-\beta)) Q\left(\frac{\frac{\mu^{(t)}}{2} - \log\left(\frac{(1-\alpha\beta)\kappa}{(1-\kappa-\alpha\kappa+\alpha\kappa\beta)}\right)}{\sqrt{\mu^{(t)}}}\right).
\end{eqnarray*}
We can show, by a calculation similar to the one followed in the proof of Theorem \ref{pr:weakrecovery}, that
\begin{eqnarray*}
\lefteqn{\lim_{b\to \infty}\lim_{n\to\infty}\frac{np_e^{\beta}}{K}}\\
&\le& \biggl(\alpha\sqrt{\beta(1-\beta)} + \\
&&\sqrt{(1-\alpha\beta) (\frac{1-\kappa}{\kappa} - \alpha(1-\beta))}\biggr) e^{-\frac{\lambda \alpha \beta^2 (1-\kappa)}{8\kappa}}.
\end{eqnarray*}
Finally by a similar calculation to (\ref{eq:boundexp}),
\begin{eqnarray*}
\lefteqn{\lim_{b\to \infty}\lim_{n\to\infty}\frac{\mathbb E (|S\Delta\widehat{S}|)}{K}} \\
&\le& 2\biggl(\alpha\sqrt{\beta(1-\beta)} + \\
&&\sqrt{(1-\alpha\beta) (\frac{1-\kappa}{\kappa} - \alpha(1-\beta))}\biggr) e^{-\frac{\lambda \alpha \beta^2 (1-\kappa)}{8\kappa}}.
\end{eqnarray*}\end{proof}

\tableofcontents

\end{document}